\documentclass[]{article}
\usepackage{amsfonts,amssymb,amsmath,epsfig,amsthm}
\usepackage{fullpage}
\usepackage{color}
\usepackage{array}

\newcommand{\be}{\begin{equation}}
\newcommand{\ee}{\end{equation}}
\newcommand{\ba}{\left[ \begin{array}}
\newcommand{\ea}{\end{array} \right]}
\newcommand{\bea}{\begin{eqnarray}}
\newcommand{\eea}{\end{eqnarray}}
\newcommand{\bc}{\begin{cases}}
\newcommand{\ec}{\end{cases}}

\def\real{\mathbb{R}}
\def\X{{\bf{X}}}

\def\w{\omega}
\def\hw{{\widehat\w}}
\def\ww{\tilde\w}

\def\I{\mathcal{I}}
\def\V{{v}}

\def\gw{\tilde{g}}
\def\Xw{\tilde{X}}
\def\Rw{\tilde{R}}
\def\Tw{\tilde{T}}
\def\1{^{\prime}}

\def\g{g}
\def\inv{^{-1}}
\def\RR{\mathbb{R}}
\def\s{\sigma}
\def\sit{\s_{i}}
\def\dX{\delta X}
\def\subs{\subset}

\def\SO{\operatorname{SO}}
\def\SE{\operatorname{SE}}
\def\so{\mathfrak{so}}
\def\se{\mathfrak{se}}
\def\Aa{\mathcal{A}}
\def\ignore#1{}
\def\imu{_\mathrm{imu}}
\def\m{{m}}
\def\M{{M}}
\newtheorem{defn}{Definition}

\newtheorem{lemma}{Lemma}

\newtheorem{claim}{Claim}

\def\cut#1{{}}
\def\margincut#1{{}}

\begin{document}

\title{\bf Observability, Identifiability and Sensitivity \\ of Vision-Aided Inertial Navigation} 
\author{Joshua Hernandez \and Konstantine Tsotsos \and Stefano Soatto}
\date{UCLA Technical Report UCLA CSD13022\\ August 20, 2013; revised May 10, 2014, September 20, 2014}

\maketitle

\begin{abstract}
We analyze the observability of 3-D pose from the fusion of visual and inertial sensors. Because the model contains unknown parameters, such as sensor biases, the problem is usually cast as a mixed filtering/identification, with the resulting observability analysis providing necessary conditions for convergence to a unique point estimate. Most models treat sensor bias rates as ``noise,'' independent of other states, including biases themselves, an assumption that is patently violated in practice. We show that, when this assumption is lifted, the resulting model is not observable, and therefore existing analyses cannot be used to conclude that the set of states that are indistinguishable from the measurements is a singleton. In other words, the resulting model is not observable. We therefore re-cast the analysis as one of sensitivity: Rather than attempting to prove that the set of indistinguishable trajectories is a singleton, we derive bounds on its volume, as a function of characteristics of the sensor and other sufficient excitation conditions. This provides an explicit characterization of the indistinguishable set that can be used for analysis and validation purposes.
\end{abstract}

%\tableofcontentsindistinguishable
\section{Introduction}

\footnotetext[1]{The authors are with the UCLA Vision Lab, University of California, Los Angeles, USA. Email: {\{jheez,ktsotsos,soatto\}@ucla.edu}.}

We present a novel approach to the analysis of observability/identifiability of three-dimensional (3-D) pose in visually-assisted navigation, whereby inertial sensors (accelerometers and gyrometers) are used in conjunction with optical sensors (vision) to yield an estimate of the 3-D position and orientation of the sensor platform. It is customary to frame this as a filtering problem, where the time-series of positions and orientations of the sensor platform is modeled as the state trajectory of a dynamical system, that produces sensor measurements as outputs, up to some uncertainty. Observability/identifiability analysis refers to the characterization of the set of possible state trajectories that produce the same measurements, and therefore are “indistinguishable” given the outputs \cite{soatto1994observability,kellyS09,mourikisR07,jonesS07,martinelli2014foundations}. 

The parameters in the model are either treated as unknown constants (e.g., calibration parameters) or as random processes (e.g., accelerometer and gyro biases) and included in the state of the model, which is then driven by some kind of {\em uninformative} (``noise'') input. Because noise does not affect the observability of a model\cut{ (by assumption, it is uninformative of the state)}, for the purpose of analysis it is usually set to zero. However, the input to the model of accelerometer and gyro bias is typically {\em small} but {\em not independent} of the state.\cut{ In fact, far from being uninformative, it is strongly correlated with it.} Thus, it should be treated as an {\em unknown input}, which is known to be ``small'' in some sense, rather than ``noise.'' 

{\em Our first contribution is to show that while (a prototypical model of) assisted navigation is {\em observable} in the absence of unknown inputs, it is {\em not} observable when unknown inputs are taken into account. }

{\em Our second contribution} is to reframe observability as a {\em sensitivity} analysis, and to show that while the set of indistinguishable trajectories is {\em not} a singleton (as it would be if the model was observable), it is nevertheless bounded. {\em We explicitly characterize this set and bound its volume} as a function of the characteristics of the inputs, which include sensor characteristics (bias rates) and the motion undergone by the platform (sufficient excitation).
\cut{
Rather than study observability of linearized system, or algebraically checking rank conditions that offer little insight on the structure of the indistinguishable states, we characterize observability directly in terms of indistinguishable sets.
}
\subsection*{Related work}

In addition to the above-referenced work on visual-inertial observability, our work relates to general unknown-input observability of linear time-invariant systems addressed in \cite{basile1969observability,hamano1983unknown}, for affine systems \cite{hammouri2010unknown}, and non-linear systems in \cite{dimassi2010robust,liberzon2012invertibility,bezzaoucha2011unknown}. The literature on robust filtering and robust identification is relevant, if the unknown input is treated as a disturbance. However, the form of the models involved in aided navigation does not fit in the classes treated in the literature above, which motivates our analysis. The model we employ includes alignment parameters for the (unknown) pose of the inertial sensor relative to the camera.

\subsection{Notation}

We adopt the notation of \cite{murrayLS94}, where a reference frame is represented by an orthogonal $3\times 3$ positive-determinant (rotation) matrix $R \in \SO(3) \doteq \{ R \in \real^{3\times 3} \ | \ R^T R = R R^T = I, \ {\rm det}(R) = +1\}$ and a translation vector $T \in \real^3$. They are collectively indicated by $g = (R, T) \in \SE(3)$. When $g$ represents the change of coordinates from a reference frame ``$a$'' to another (``$b$''), it is indicated by $g_{ba}$. Then the columns of $R_{ba}$ are the coordinate axes of $a$ relative to the reference frame $b$, and $T_{ba}$ is the origin of $a$ in the reference frame $b$. If $p_a$ is a point relative to the reference frame $a$, then its representation relative to $b$ is $p_b = g_{ba} p_a$. In coordinates, if $X_a$ are the coordinates of $p_a$, then $X_b = R_{ba}X_a + T_{ba}$ are the coordinates of $p_b$. 

A time-varying pose is indicated with $g(t) = (R(t), T(t))$ or $g_t = (R_t, T_t)$, and the entire trajectory from an initial time $t_i$ and a final time $t_f$ $\{g(t) \}_{t = t_i}^{t_f}$ is indicated in short-hand notation with $g_{t_i}^{t_f}$; when the initial time is $t_0 = 0$, we omit the subscript and call $g^{t}$ the trajectory ``up to time $t$''. The time-index is sometimes omitted for simplicity of notation when it is clear from the context.

We indicate with $\widehat V = (\widehat \w, v) \in \se(3)$ the (generalized) velocity or ``twist'', where $\widehat \w$ is a skew-symmetric matrix $\widehat \w \in \so(3) \doteq \{S \in \real^{3\times 3} \ | \ S^T = -S\}$ corresponding to the cross product with the vector $\w \in \real^3$, so that $\widehat \w v = \w \times v$ for any vector $v\in \real^3$. We indicate the generalized velocity with $V = (\w, v)$. We indicate the group composition $g_1 \circ g_2$ simply as $g_1 g_2$. 
In homogeneous coordinates, $\bar X_b = G_{ba} \bar X_a$ where $\bar X^T = [X^T \ 1]$ and 
\begin{equation}
G \doteq \ba{cc} R & T \\ 0 & 1 \ea \in \real^{4\times 4}
~~~~
\hat V \doteq \ba{cc} \widehat \w & v \\ 0 & 0 \ea.\nonumber
\end{equation}
Composition of rigid motions is then represented by matrix product.

\subsection{Mechanization Equations}

The motion of a sensor platform is represented as the time-varying pose $g_{sb}$ of the body relative to the spatial frame. To relate this to measurements of an inertial measurement unit (IMU) we compute the temporal derivatives of $g_{sb}$, which yield the (generalized) body velocity $V_{sb}^b$, defined by $\dot g_{sb}(t) = g_{sb}(t) {\widehat V}^b_{sb}(t)$, which can be broken down into the rotational and translational components $\dot R_{sb}(t) = R_{sb}(t) \widehat{\w}_{sb}^b(t)$ and $\dot T_{sb}(t) = R_{sb}(t) v_{sb}^b(t)$. An ideal gyrometer (gyro) would measure $\w\imu  = \w_{sb}^b$. The translational component of body velocity, $v_{sb}^b$, can be obtained from the last column of the matrix $\frac{d}{dt} {\widehat V}^b_{sb}(t)$. That is, $ \dot{v}_{sb}^b = \dot{R_{sb}^T}\dot T_{sb} + R_{sb}^T \ddot T_{sb} = - \widehat{\w}_{sb}^b v_{sb}^b + R_{sb}^T \ddot T_{sb} \doteq - \widehat{\w}_{sb}^b v_{sb}^b + \alpha_{sb}^b $, which serves to define $\alpha_{sb}^b \doteq R_{sb}^T \ddot T_{sb}
$. These equations can be simplified by defining a new linear velocity, $v_{sb}$, which is neither the body velocity $v_{sb}^b$ nor the spatial velocity $v_{sb}^s$, but instead $v_{sb} \doteq R_{sb}v_{sb}^b$. Consequently, we have that $ \dot T_{sb}(t) = v_{sb}(t) $ and $ \dot v_{sb}(t) = \dot R_{sb} v_{sb}^b + R_{sb} \dot{v}_{sb}^b = \ddot T_{sb} \doteq \alpha_{sb}(t) $ where the last equation serves to define the new linear acceleration $\alpha_{sb}$; as one can easily verify, we have that $ \alpha_{sb} = R_{sb} \alpha_{sb}^b.$ An ideal accelerometer (accel) would then measure  $ \alpha\imu  = R_{sb}^T(t) (\alpha_{sb}(t) - \gamma)$ where $\gamma \in \real^3$ is the gravity vector.

There are several reference frames to be considered in a navigation scenario. The {\em spatial frame} $s$, typically attached to Earth and oriented so that gravity $\gamma$ takes the form $\gamma^T = [0 \ 0 \ 1]^T \| \gamma \|$ where $\| \gamma \|$ can be read from tabulates based on location\cut{\footnote{Although see App. \ref{app-gravity}} for when such information is unavailable.} and is typically around $9.8m/s^2$. The {\em body frame} $b$ is attached to the IMU.\cut{\footnote{In practice, the IMU has several different frames due to the fact that the gyro and accel are not co-located and aligned, and even each sensor (gyro or accel) is composed of multiple sensors, each of which can have its own reference frame. Here we will assume that the IMU has been pre-calibrated so that accel and gyro yield measurements relative to a common reference frame, the {\em body frame}. In reality, it may be necessary to calibrate the alignment between the multiple-axes sensors (non-orthogonality), as well as the gains (scale factors) of each axis.}} The {\em camera frame} $c$, relative to 
which image 
measurements are captured, is also unknown, although we will assume that {\em intrinsic calibration} has been performed, so that measurements on the image plane are provided in metric units. 

The equations of motion (known as mechanization equations) are usually described in terms of the body frame at time $t$ relative to the spatial frame $g_{sb}(t)$. Since the spatial frame is arbitrary (other than for being aligned to gravity), it is often chosen to be co-located with the body frame at time $t = 0$. To simplify the notation, we indicate this time-varying frame $g_{sb}(t)$ simply as $g$, and so for $R_{sb}, T_{sb}, \w_{sb}, v_{sb}$, thus effectively omitting the subscript $sb$ wherever it appears. This yields  $\dot T = \V, \ \dot R = R \widehat \w, \ \dot \V = \alpha, \dot \w = w, \ \dot \alpha = \xi$
where $w \in \real^3$ is the rotational acceleration, and $\xi \in \real^3$ the translational jerk (derivative of acceleration). 
\subsection{Sensor model}
Although the acceleration $\alpha$ defined above corresponds to neither body nor spatial acceleration, it is conveniently related to accelerometer measurements $\alpha\imu$: 
\begin{equation}
{\alpha\imu (t) = R^T(t) (\alpha(t)- \gamma) + \underbrace{\alpha_b(t) + n_{\alpha}(t)}}
\label{eq-accel}
\end{equation}
where the measurement error in bracket includes a slowly-varying mean (``bias'') $\alpha_b(t)$ and a residual term $n_\alpha$ that is commonly modeled as a zero-mean (its mean is captured by the bias), white, homoscedastic and Gaussian noise process. In other words, it is assumed that $n_\alpha$ is independent of $\alpha$, hence uninformative. Here $\gamma$ is the gravity vector expressed in the spatial frame.\cut{\footnote{The orientation of the body frame relative to gravity, $R_0$, is unknown, but can be approximated by keeping the IMU still (so $R^T(t) = R_0$) and averaging the accel measurements, so that $\frac{1}{T}\sum_{t=0}^T \alpha\imu (t)  \simeq  - R_0^T \gamma + \alpha_b$. Assuming the bias to be small (zero), this equation defines $R_0$ up to a rotation around gravity, which is arbitrary. Note that if $\alpha_b \neq 0$, the initial bias will affect the initial orientation estimate.}} Measurements from a gyro, $\w\imu$, can be similarly modeled as 
\begin{equation}
{\w\imu (t) = \w(t) + \underbrace{\w_b(t) + n_{\w}(t)}}
\label{eq-gyro}
\end{equation}
where the measurement error in bracket includes a slowly-varying bias $\w_b(t)$ and a residual ``noise'' $n_\w$ also assumed zero-mean, white, homoscedastic and Gaussian, independent of $\w$.

Other than the fact that the biases $\alpha_b, \w_b$ change {\em slowly}, they can change arbitrarily. One can therefore consider them an {\em unknown input} to the model, or a {\em state} in the model, in which case one has to hypothesize a dynamical model for them. For instance,
\begin{equation}
\dot \w_b(t) = w_b(t), ~~~ \dot \alpha_b(t) = \xi_b(t)
\end{equation}
for some unknown inputs $w_b, \xi_b$ that can be safely assumed to be {\em small}, but not (white, zero-mean and, most importantly) independent of the biases. Nevertheless, it is common to consider them to be realizations of a Brownian motion that is  {\em independent} of $\w_b, \alpha_b$. This is done for convenience as one can then consider all unknown inputs as ``noise.'' Unfortunately, however, this has implications on the analysis of the observability and identifiability of the resulting model.

\subsection{Model reduction} 

The mechanization equations above define a dynamical model having as output the IMU measurements. Including the initial conditions and biases, we have %\margincut{notation inconsistency $n_\omega$ and $v_\omega$}
\begin{equation}
\begin{cases}
\begin{tabular}{>{$}r<{$} >{$\!\!\!\!\!}l<{$} >{$}r<{$} >{$\!\!\!\!\!}l<{$}}
\dot T &= \V & T(0) &= 0 \\
\dot R &= R \widehat \w & R(0) &= R_0\\
\dot \V &= \alpha \\ %~~~~~~ \V(0) = \V_0 \\
\dot \w &= w\\ % ~~~~~~~ \w(0) = \w_0 \\
\dot \alpha &= \xi \\% ~~~~~~ \alpha(0) = \alpha_0 \\
\dot \w_b &= w_b  \\
\dot \alpha_b &= \xi_b  \\ 
\dot \gamma &= 0 \\
\end{tabular}\\
\begin{tabular}{>{$}r<{$} >{$\!\!\!\!\!}l<{$}}
\w\imu (t) &= \w(t) + \w_b(t) + n_{\w}(t) \\ 
\alpha\imu (t) &= R^T(t) (\alpha(t)- \gamma) + \alpha_b(t) + n_{\alpha}(t) 
\end{tabular}
\end{cases}
\label{eq-standard}
\end{equation}
In this standard model, data from the IMU are considered as (output) {\em measurements}. However, it is customary to treat them as (known) {\em input} to the system, by writing $\w$ in terms of $\w\imu $ and $\alpha$ in terms of $\alpha\imu $:
\begin{equation}
{\w = \w\imu  - \w_b + \underbrace{n_R}_{-n_{\w}}} ~~~~~~ {\alpha = R(\alpha\imu  - \alpha_b) + \gamma + \underbrace{ n_\V}_{- Rn_\alpha}}
\end{equation}
This equality is valid for {\em samples} (realizations) of the stochastic processes involved, but it can be misleading as, if considered as stochastic processes, the noises above are {\em not} independent of the states. Such a dependency is nevertheless typically neglected. The resulting mechanization model is
\begin{equation}
%\boxed{
\begin{cases}
\begin{tabular}{>{$}r<{$} >{$\!\!\!\!\!}l<{$} >{$}r<{$} >{$\!\!\!\!\!}l<{$}}
\dot T &= \V &T(0) &= 0 \\
\dot R &= R (\widehat \w\imu  - \widehat \w_b) + n_{R} &R(0) &= R_0\\
\dot \V &= R(\alpha\imu  - \alpha_b) + \gamma + n_\V \\ %~~~~~~ \V(0) = \V_0 \\
\dot \w_b &= w_b \\
\dot \alpha_b &= \xi_b. 
\end{tabular}
\end{cases}
%}
%\label{eq-mech}
\label{eq-model-dyn}
\end{equation}
\subsection{Imaging model and alignment}
Initially we assume there is a collection of points $X^i, \ i = 1, \dots, N$, visible from time $t=0$ to the current time $t$. If $\pi:\real^3 \rightarrow \real^2; X \mapsto [X_1/X_3, \ X_2/X_3]$ is a canonical central (perspective) projection, assuming that the camera is {\em calibrated},\footnote{Intrinsic calibration parameters are known and compensated for.} {\em aligned},\footnote{The pose of the camera relative to the IMU is known and compensated for.} and that the spatial frame coincides with the body frame at time $0$, we have
\begin{equation}
y^i(t) = {\color{black} \frac{R^T_{1:2}(t) (X^i - T_{1:2}(t))}{R^T_{3}(t)( X^i - T_3(t))}} \doteq  \pi(g^{-1}(t)X^i) + n^i(t)
\label{eq-y}
\end{equation}
%% \begin{equation}
%% \boxed{y^i(t) = {\color{black} \frac{R^T_{1:2}(t) (X^i - T_{1:2}(t))}{R^T_{3}(t)( X^i - T_3(t))}} \doteq  \pi(g^{-1}(t)X^i) + n^i(t), ~~~ {\color{black} t \ge 0.}}
%% \label{eq-y}
%% \end{equation}
If the feature first appears at time $t = 0$ {\em and if the camera reference frame is chosen to be the origin the world reference frame} so that $T(0) = 0; R(0) = I$, then we have that $y^i(0) = \pi(X^i) +n^i(0)$, and therefore
\begin{equation}
{X^i = \bar y^i(0)Z^i + \tilde n^i}
\label{eq-trian}
\end{equation}
where $\bar y$ is the homogeneous coordinate of $y$, $\bar y = [y^T \ 1]^T$, and $\tilde n^i = [{n^i}^T(0)Z^i \ \  0]^T$. Here $Z^i$ is the (unknown, scalar) depth of the point at time $t = 0$, and again the dependency of the noise on the state is neglected. With an abuse of notation, we write the map that collectively projects all points to their corresponding locations on the image plane as $y(t) = \pi(g^{-1}(t) \X) + n(t)$, or:
\begin{equation}
y(t) \doteq \ba{c}
y^1 \\ y^2 \\ \vdots \\ y^N \ea (t)
 = \ba{c}
\pi(R^T(X^1 - T)) \\
\pi(R^T(X^2- T)) \\
\vdots\\
\pi(R^T( X^N- T))
\ea
+ \ba{c}
n^1(t) \\ n^2(t) \\ \vdots \\ n^N(t) \ea 
\label{eq-vis}
\end{equation}

%\subsubsection{Alignment (calibration)} 

In practice, the measurements $y(t)$ are known only up to a transformation $g_{cb}$  mapping the body frame to the camera, often referred to as ``alignment'': 
\begin{equation}
{y^i(t) = \pi\left( g_{cb} g^{-1}(t) X_s^i \right) + n^i(t) \in \real^2}
\end{equation}
We can then, as done for the points $X^i$, add it to the state with trivial dynamics $\dot g_{cb} = 0$.

%\subsubsection{Groups (occlusions)}

It may be convenient in some cases to represent the points $X_s^i$ in the reference frame where they first appear, say at time $t_i$, rather than in the spatial frame. This is because the uncertainty is highly structured in the frame where they first appear: if $X^i(t_i) = \bar y^i(t_i) Z^i(t_i)$, then $y^i(t_i)$ has the same uncertainty of the feature detector (small and isotropic on the image plane) and $Z^i$ has a large uncertainty, but it is constrained to be positive. 

However, to relate $X^i(t_i)$ to the state, we must bring it to the spatial frame, via $g(t_i)$, which is unknown. Although we may have a good approximation of it, the current estimate of the state $\hat g(t_i)$, the pose when the point first appears should be estimated along with the coordinates of the points. Therefore, we can represent $X^i$ using $y^i(t_i)$, $Z^i(t_i)$ {\em and} $g(t_i)$: 
\begin{equation}
X_s^i = X_s^i(g_{t_i}, y_{t_i}, Z_{t_i}) = g_{t_i} \bar y_{t_i} Z_{t_i}
\end{equation}
Clearly this is an over-parametrization, since each point is now represented by $3+6$ parameters instead of $3$. However, the pose $g_{t_i}$ can be pooled among all points that appear at time $t_i$, considered therefore as a {\em group}. At each time, there may be a number $j = 1, \dots, K(t)$ groups, each of which has a number $i = 1, \dots, N_j(t)$ points. We indicate the group index with $j$ and the point index with $i = i(j)$, omitting the dependency on $j$ for simplicity. The representation of $X_s^i$ then evolves according to
\begin{equation}
\begin{cases}
\begin{tabular}{>{$}r<{$} >{$\!\!\!\!\!}l<{$}}
\dot y^i_{t_i} &= 0, ~~~ i = 1, \dots, N(j) \\
\dot Z^i_{t_i} &= 0 \\
\dot g_{j} &= 0, ~~~~ j = 1, \dots, K(t).
\end{tabular}
\end{cases}
\label{eq-groups}
\end{equation}

\section{Analysis of the model} 
\label{sect-analysis}

The goal here is to exploit imaging and inertial measurements to infer the sensor platform trajectory. For this problem to be well-posed, a (sufficiently exciting) realization of $\w\imu, \alpha\imu$ and $y$ should constrain the set of trajectories that satisfy \eqref{eq-model-dyn}-\eqref{eq-groups} to be unique. If there are different trajectories that satisfy \eqref{eq-standard} with the same outputs, they are {\em indistinguishable}. If the set of indistinguishable trajectories is a singleton (contains only one element, presumably the ``true'' trajectory), the model \eqref{eq-standard} is {\em observable}, and one may be able to retrieve a unique point-estimate of the state using a filter, or observer.

While it is commonly accepted that the model \eqref{eq-standard} or its equivalent reduced realization, is observable, this is the case only when {\em biases are exactly constant.} But if biases are allowed to change, however slowly, the observability analysis conducted thus far cannot be used to conclude that the indistinguishable set is a singleton. Indeed, we show that this is the not the case, by computing the indistinguishable set explicitly. The following claim\cut{, first stated in \cite{hernandezS13}} is proven in \cite{hernandezS13}. 
\begin{claim}[Indistinguishable Trajectories]\label{claim-five}
Let $g(t)= (R(t), T(t)) \in \SE(3)$ satisfy \eqref{eq-model-dyn}-\eqref{eq-groups}
for some known constant $\gamma$ and functions $\alpha\imu (t)$, $\w\imu (t)$ and for some unknown functions $\alpha_b(t), \w_b(t)$ that are constrained to have $\| \dot \alpha_b(t) \| \le \epsilon$, $\| \dot \w_b(t) \| \le \epsilon$, and $\|\ddot\w_b(t)\|\le\epsilon$ at all $t$,
for some $\epsilon<1$.

Suppose $\gw(t) \doteq \sigma(g_B g(t) g_A)$ for some  constant $g_A = (R_A, T_A)$, $g_B = (R_B, T_B)$, $\sigma > 0$,
with bounds on the configuration space such that\footnote{Here $\sigma(g)$ is a scaled rigid motion: if $g = (R, T)$, then $\sigma(g) =  (R, \sigma T)$.}
$\|T_A\|\leq M_A$ and $0<m_\sigma\leq|\sigma|\leq M_\sigma$. 
Then, under sufficient excitation conditions,
$\gw(t)$ satisfies (\ref{eq-model-dyn})-\eqref{eq-groups} if and only if 
\begin{gather}
\| I - R_A \|  \leq  \frac{2{\epsilon}}{\m(\dot{\w}\imu\!:\!{\RR^+})}  \label{constraint1}\\
|\sigma - 1|  \le \frac{k_{1}\epsilon + M_\sigma\|I-R_A\|}{\m(\dot\alpha\imu\!:\!{\I_{1}})} \label{constraint2}\\
\|T_A\|\leq \frac{\epsilon(k_{2}+(2M_\sigma+1)M_A)}{m_{\sigma}\,\m(\ddot\w\imu\!:\!{\I_{2}})}\label{constraint3} \\
\|(1-R_B^T)\gamma\|\leq  \frac{\epsilon(k_{3} + M_\sigma M_A)}
{m_{\sigma}\,\m(\w\imu-\w_b\!:\!\I_{3})} + \nonumber \\ + \frac{(|\sigma-1|+\epsilon)\M(\w\imu-\w_b\!:\!\I_{3})\|\gamma\|}
{m_{\sigma}\,\m(\w\imu-\w_b\!:\!\I_{3})} \label{constraint4}
\end{gather}
%% \begin{gather}
%% \| I - R_A \|  \leq  \frac{2{\epsilon}}{\m(\dot{\w}\imu\!:\!{\RR^+})}  \label{constraint1}\\
%% |\sigma - 1|  \le \frac{k_{1}\epsilon + M_\sigma\|I-R_A\|}{\m(\dot\alpha\imu\!:\!{\I_{1}})} \label{constraint2}\\
%% \|T_A\|\leq \frac{\epsilon(k_{2}+(2M_\sigma+1)M_A)}{m_{\sigma}\,\m(\ddot\w\imu\!:\!{\I_{2}})}\label{constraint3} \\
%% \|(1-R_B^T)\gamma\|\leq  \frac{\epsilon(k_{3} + M_\sigma M_A) + (|\sigma-1|+\epsilon)\M(\w\imu-\w_b\!:\!\I_{3})\|\gamma\|}
%% {m_{\sigma}\,\m(\w\imu-\w_b\!:\!\I_{3})}\label{constraint4}
%% \end{gather}
for $\I_i$ and $k_i$ determined by the sufficient excitation conditions.
\end{claim}
The set of indistinguishable trajectories in the limit where $\epsilon \rightarrow 0$ is parametrized by an arbitrary $T_B\in \real^3$ and $\theta \in \real$, 
\begin{equation}
\begin{cases}
%\Xw^i = X^i \\
\tilde T =  \exp(\widehat \gamma \theta) T + T_B   \\
\tilde R = \exp(\widehat \gamma \theta) R   \\ 
\tilde T_{t_i} = \exp(\widehat \gamma \theta) \bar T_{t_i} + T_B \\
\tilde R_{t_i} = \exp(\widehat \gamma \theta) \bar R_{t_i}  ~~~~~~~~ {\rm up \ to \ } 
{\cal O}\left( 
\frac{\| \dot \w_b \|}{\|\dot \w\imu  \|}, \frac{\| \dot \alpha_b \|}{\|\dot \alpha\imu \|},  \frac{1}{\| \gamma \|}
\right)  \\
\tilde T_{cb} = T_{cb}  \\
\tilde R_{cb} = R_{cb}   
\end{cases}
\label{eq-obs-zeroinput}
\end{equation}
If we impose that $T(0) = \tilde T(0) = 0$, then $T_B = 0$ is determined; similarly, if we impose the initial pose to be aligned with gravity (so gravity is in the form $[0 \ 0 \ \| \gamma \| ]^T)$, then $\theta = 0$. But while we can impose this condition, we cannot {\em enforce} it, since the initial condition is not a part of the state of the filter, so we cannot relate the measurements at each time $t$ directly to it.

However, if the reference can be associated to {\em constant parameters} that are part of the state of the model, it can be enforced in a consistent manner. For instance, the ambiguous set of points is
\begin{equation}
\Xw^j = g_a \bar g_i^{-1} g_i g_a^{-1} X^j,
\end{equation}
if each group $i$ contains at least $3$ non-coplanar points, it is possible to fix $\bar g_i$ by parameterizing $X^j \doteq \bar y^j_{t_i} Z^j$ and imposing three directions $y^j_{t_i} = {\tilde y}^j_{t_i} = y^j(t_i), j = 1, \dots, 3$, the measurement of these directions at time $t_i$ when they first appear. This yields $\bar g_i = g_i$ and $\Xw^j = X^j$ for that group. Note that it is necessary to impose this constraint in {\em each group}.

The residual set of indistinguishable trajectories is parameterized by {\em constants} $\theta, T_B$, that determine a Gauge transformation for the groups, that can be fixed by always fixing the pose of {\em one} of the groups. This can be done in a number of ways. For instance, if for a certain group of points indexed by $i$ we impose
\begin{equation}
R_{t_i}  = \tilde R_{t_i} = \hat R(t_i) \ {\rm and} \ T_{t_i}  = \tilde T_{t_i} = \hat T(t_i)
\end{equation}
by assigning their value to the current best estimate of pose and not including the corresponding variables in the state of the model, then we have that
\begin{equation}
\hat R(t_i) = \exp(\widehat \gamma \theta) \hat R(t_i) 
\end{equation}
and therefore $\theta = 0$; similarly, 
\begin{equation}
T_B = (I - \exp(\widehat \gamma \theta))T(t_i) = 0.
\end{equation}
 Therefore, the gauge transformation is enforced explicitly at each instant of time, as each measurement provides a constraint on the states. After the Gauge Transformation has been fixed, the model is observable in the limit $\epsilon \rightarrow 0$, and otherwise the state of an observer is related to the true one as follows: 
\begin{eqnarray}
{\Xw}^{\rm ref} &=&   (1+\tilde  \sigma)\tilde R_{cb} e^{\w_B} e^{\widehat \gamma \theta}e^{\w_A} \tilde R_{cb}^T (X^{\rm ref}-T_A) + \nonumber \\ && + (1+\tilde \sigma)( \tilde R_{cb} e^{\w_A}T_B + \tilde R_{cb} T_A + \tilde T_{cb})  \\
{\Xw}^j &=&   (1+\tilde  \sigma)\tilde R_{cb} \bar R_i \tilde R_{t_i}  \tilde R_{cb}^T (X^{j}-T_A) + \nonumber \\ && + (1+\tilde \sigma)( \tilde R_{cb} \bar R_i \tilde T_{t_i} + \tilde R_{cb} \bar T_{i} + \tilde T_{cb}) \\
\tilde T %&=&  (R_B T + T_B + R_B R T_A)(1+\tilde \sigma) \\ 
&=& e^{\widehat \gamma \theta} T + T_B (1+\tilde\sigma) + \nonumber \\  && + \w_Be^{\widehat \gamma \theta} T + e^{\w_B} e^{\widehat \gamma \theta}R T_A(1+\tilde \sigma) \\
\tilde R &=&  e^{\w_B}e^{\widehat \gamma \theta} R e^{\w_A}\\
\tilde T_{t_i} &=&  e^{\widehat \gamma \theta} \bar T_{i} + T_B (1+\tilde\sigma) + \nonumber  \\ && + \w_Be^{\widehat \gamma \theta} \bar T_{i} + e^{\w_B} e^{\widehat \gamma \theta}\bar R_{i} T_A(1+\tilde \sigma)\\
\tilde R_{t_i} &=&  e^{\w_B}e^{\widehat \gamma \theta} \bar R_{i} e^{\w_A}\\
\tilde T_{cb} &=&  T_{cb} + \tilde \sigma T_{cb} + R_{cb}T_A(1+ \tilde \sigma) \\
\tilde R_{cb} &=&  R_{cb}\exp(\w_A)
%\tilde \alpha_b &=&  (\ref{eq-alphatilde}) \nonumber  \\
%\tilde \w_b &=& (\ref{eq-omtilde}) \nonumber \\
\end{eqnarray}
%% \begin{eqnarray}
%% {\Xw}^{\rm ref} &=&   (1+\tilde  \sigma)\tilde R_{cb} e^{\w_B} e^{\widehat \gamma \theta}e^{\w_A} \tilde R_{cb}^T (X^{\rm ref}-T_A) +(1+\tilde \sigma)( \tilde R_{cb} e^{\w_A}T_B + \tilde R_{cb} T_A + \tilde T_{cb})  \\
%% {\Xw}^j &=&   (1+\tilde  \sigma)\tilde R_{cb} \bar R_i \tilde R_{t_i}  \tilde R_{cb}^T (X^{j}-T_A) +(1+\tilde \sigma)( \tilde R_{cb} \bar R_i \tilde T_{t_i} + \tilde R_{cb} \bar T_{i} + \tilde T_{cb}) \\
%% \tilde T %&=&  (R_B T + T_B + R_B R T_A)(1+\tilde \sigma) \\ 
%% &=& e^{\widehat \gamma \theta} T + T_B (1+\tilde\sigma) + \w_Be^{\widehat \gamma \theta} T + e^{\w_B} e^{\widehat \gamma \theta}R T_A(1+\tilde \sigma) \\
%% \tilde R &=&  e^{\w_B}e^{\widehat \gamma \theta} R e^{\w_A}\\
%% \tilde T_{t_i} &=&  e^{\widehat \gamma \theta} \bar T_{i} + T_B (1+\tilde\sigma) + \w_Be^{\widehat \gamma \theta} \bar T_{i} + e^{\w_B} e^{\widehat \gamma \theta}\bar R_{i} T_A(1+\tilde \sigma)\\
%% \tilde R_{t_i} &=&  e^{\w_B}e^{\widehat \gamma \theta} \bar R_{i} e^{\w_A}\\
%% \tilde T_{cb} &=&  T_{cb} + \tilde \sigma T_{cb} + R_{cb}T_A(1+ \tilde \sigma) \\
%% \tilde R_{cb} &=&  R_{cb}\exp(\w_A)
%% %\tilde \alpha_b &=&  (\ref{eq-alphatilde}) \nonumber  \\
%% %\tilde \w_b &=& (\ref{eq-omtilde}) \nonumber \\
%% \end{eqnarray}
where $\omega_A, \ R_A, \ \sigma, \ T_A, \ \omega_B, \ R_B$ satisfy \eqref{constraint1}-\eqref{constraint4}, 
and $\theta$, $T_B$ are arbitrary. The groups will be defined up to an arbitrary reference frame $(\bar R_i, \bar T_i)$, except for the reference group where that transformation is fixed. Note that, as the reference group ``switches'' (when points in the reference group become occluded or otherwise disappear due to failure in the data association mechanism), a small error in pose is accumulated. This error affects the gauge transformation, not the {\em state} of the system, and therefore is not reflected in the innovation, nor in the covariance of the state estimate, that remains bounded. This is unlike \cite{mourikisR07}, where the covariance of the translation state $T_B$ and the rotation about gravity $\theta$ grows unbounded over time, possibly affecting the numerical aspects of the implementation. Notice that in the limit where $\dot \w_b = \dot \alpha_b =0$, we obtain back Eq. (\ref{eq-obs-zeroinput}). Otherwise, the equations above immediately imply the following
\begin{claim}[unknown-input observability]
\label{cor-unobservable}
The model (\ref{eq-model-dyn})-\eqref{eq-groups} is {\em not} observable, even after fixing the Gauge ambiguity, as the indistinguishable set is not a singleton, unless biases are constant ($\epsilon = 0$) or their derivative is known exactly.
\end{claim}
We refer the reader to \cite{hernandezS13} for proofs, which are articulated into several steps. 
In practice, once the Gauge transformations are fixed, a properly designed filter can be designed to converge to a point estimate, but there is no guarantee that such an estimate coincides with the true trajectory. Instead, the estimate can deviate from the true trajectory depending on the biases. The analysis above quantifies how far from the true trajectory the estimated one can be, provided that the estimation algorithm uses bounds on the bias drift rates and the characteristics of the motion. Often these bounds are not strictly enforced but rather modeled through the driving noise covariance. 

\section{Empirical validation}

To validate the analysis, we run repeated trials to estimate the state of the platform under different motion but identical alignment (the camera is rigidly connected to the IMU and the connection is stable to high precision). If alignment parameters were identifiable (or the augmented state observable), we would expect convergence to the same parameters across all trials. Instead, Fig. \ref{fig-one} shows that the filter reaches steady-state, with the estimates of the parameters stabilizing, but to different values at each run. 
\begin{figure}[htb]
\center
\hspace{-0.5cm}
\includegraphics[width=0.5\textwidth]{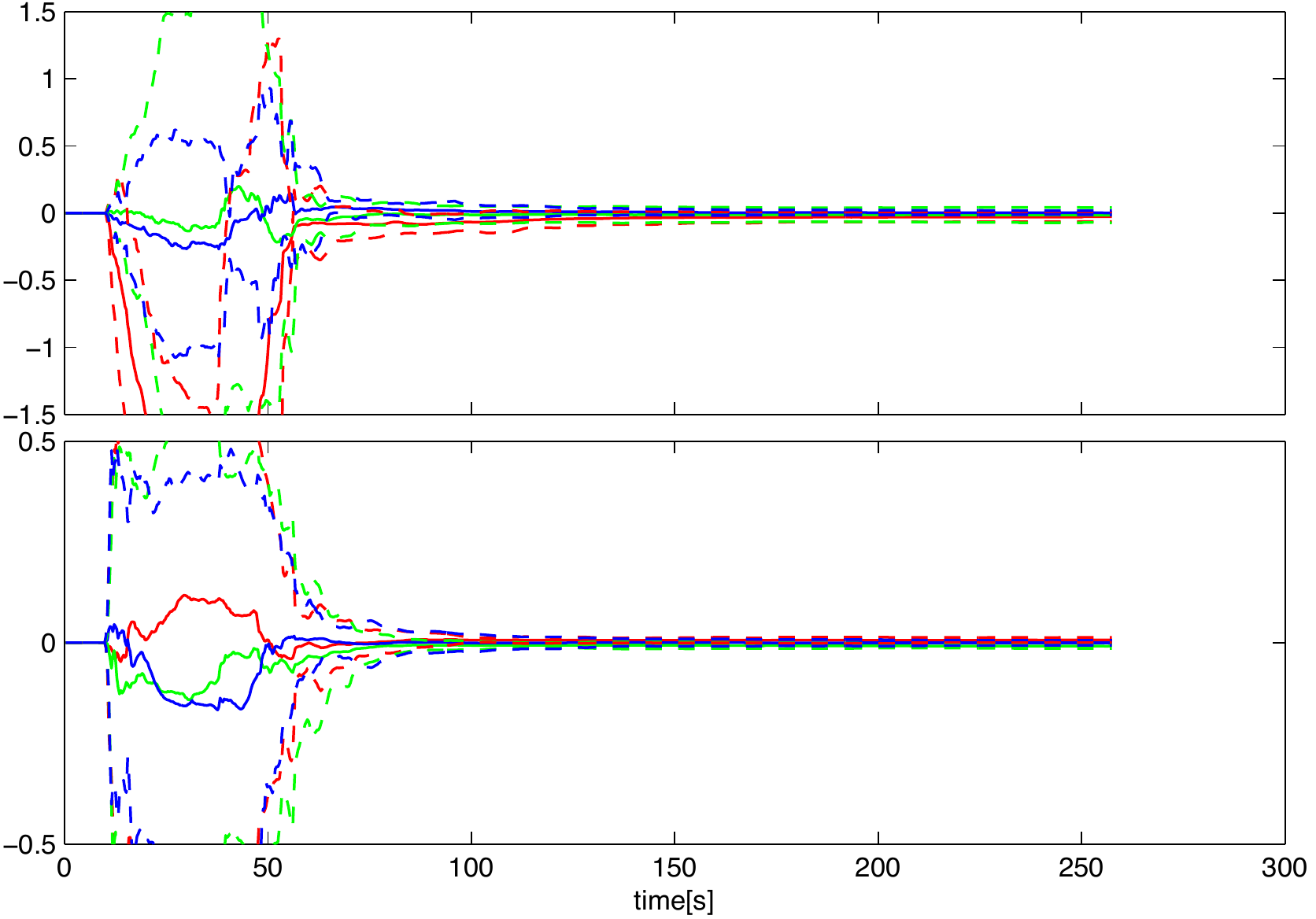}
%\vspace{6cm}
\caption{\sl Convergence of alignment parameters (top translational, bottom rotational) to a set, rather than a unique point estimate, due to the lack of unknown-input observability in the presence of non-constant biases. The mean (solid line) and twice the standard deviation (dashed lines) of the change in estimated parameters relative to their initial nominal values across multiple trials on real data collected with our experimental framework, show that different trials converge to different parameter values, but to within a bounded set. The standard deviations of the converged translational parameters (in centimeters) are [1.76 2.8 0.77] and [0.0032 0.0029 0.0033] for the rotational parameters (in radians). }
\label{fig-one}    
\end{figure}
Nevertheless, the parameter values are in a set, whose volume can be bounded based on the analysis above and the characteristics of the sensor. In particular, less stable biases, and less exciting motions, result in a larger indistinguishable set: Fig. \ref{fig-two} shows the same experiments with more gentle (hence less exciting) motions. 
\begin{figure}[htb]
\center
\hspace{-0.5cm}
\includegraphics[width=0.5\textwidth]{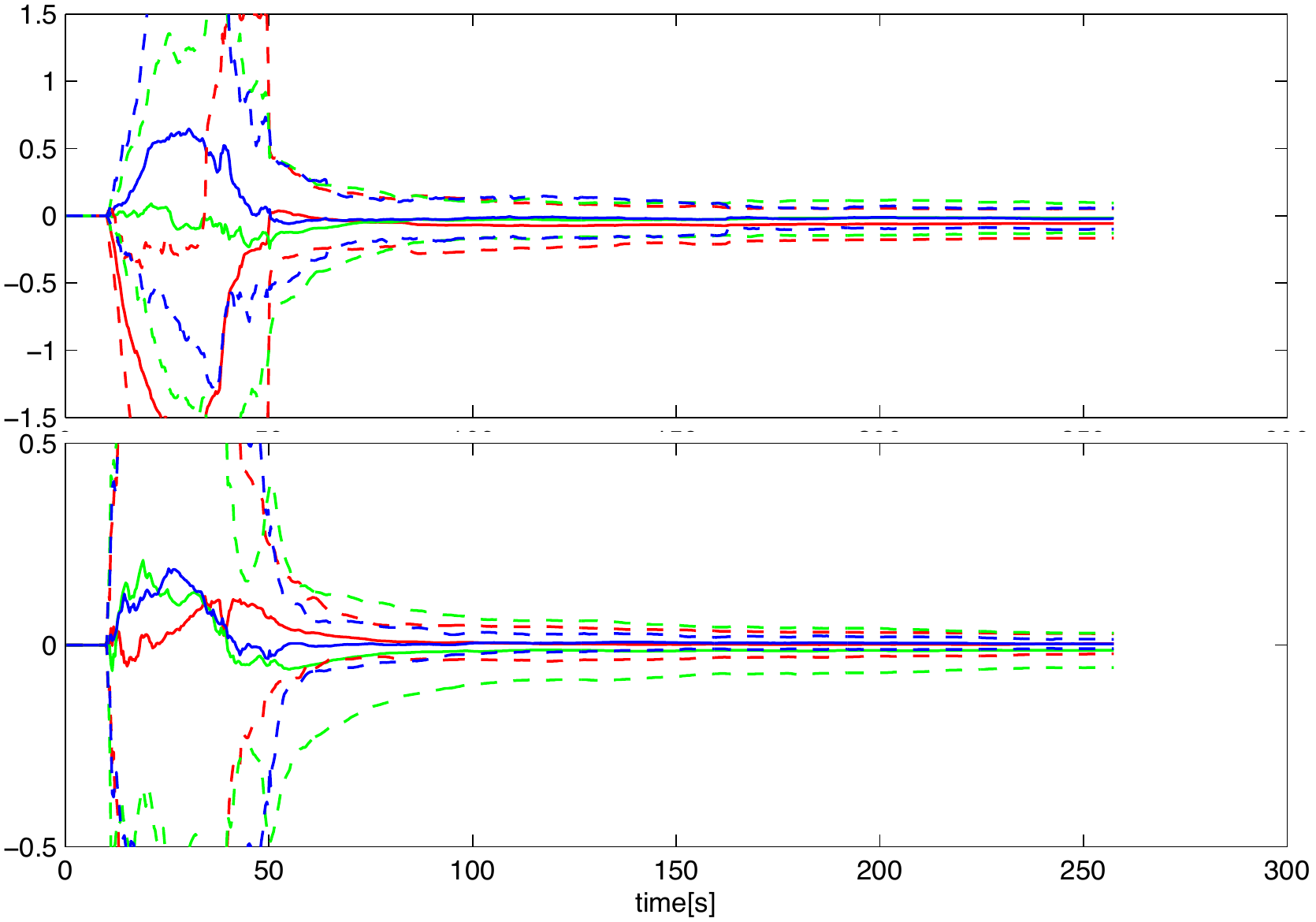}
\caption{\sl The indistinguishable set is bounded depending on the characteristic of the motion, that has to be sufficiently exciting. Gentler motion produces multiple trials that converge (top translational, bottom rotational) to a set of larger volume compared to Fig. \ref{fig-one}. The standard deviations of the converged translational parameters (in centimeters) are [4.43 5.98 3.57] and [0.0069 0.0079 0.0062] for the rotational parameters (in radians).}
  \label{fig-two}
\end{figure}
Fig. \ref{fig-three} shows the same where the accel and gyro biases have been artificially inflated by adding a slowly time-varying offset to the IMU measurements.
\begin{figure}[htb]
\center
\hspace{-0.5cm}
\includegraphics[width=0.5\textwidth]{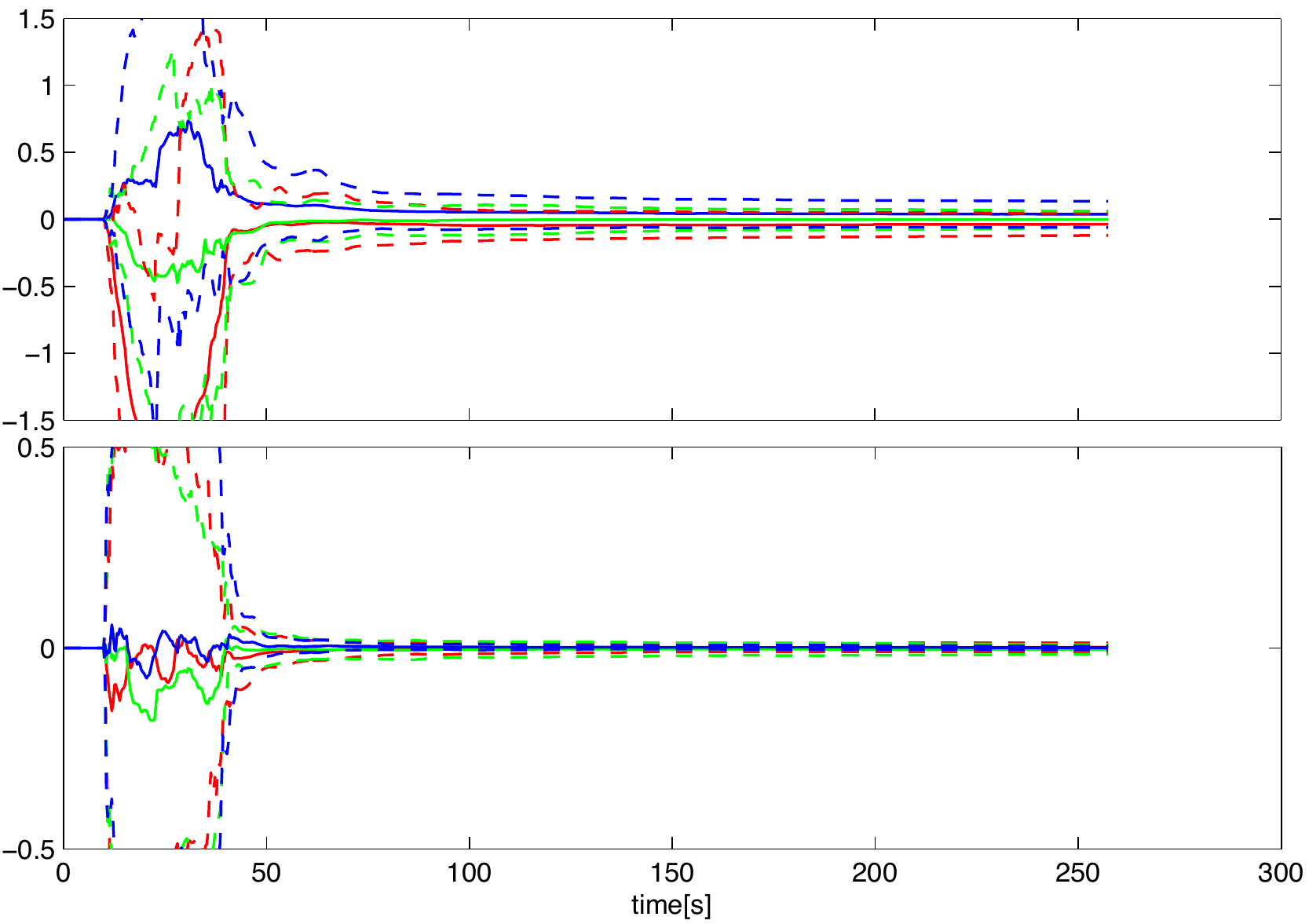}
\caption{\sl The indistinguishable set also depends on the characteristics of the sensor, and its volume is directly proportional to the sensor bias rate. Here artificial bias is added to the measurements, resulting in a larger indistinguishable set (top translational alignment, bottom rotational alignment) compared to Fig. \ref{fig-one}. The standard deviations of the converged translational parameters (in centimeters) are [4.09 3.1 4.88] and [0.0053 0.0061 0.002] for the rotational parameters (in radians).}
\label{fig-three}
\end{figure}
To further support the conclusions of the analysis, Monte-Carlo experiments were conducted on the model in simulation using stationary and time-varying biases while undergoing sufficiently exciting motion. For each trial, the platform views a consistent set of randomly generated points (no occlusions) while circling the point set on randomly generated trajectories. Figures \ref{fig-mc-statbias} and \ref{fig-mc-floatbias} show the resulting estimation errors of the alignment states for 20 trials each using a constant and white-noise driven bias respectively. As seen in the experiments with real data, estimates in the time-varying bias scenario do not converge to a singleton. 
\begin{figure}[htb]
\center
\hspace{-0.5cm}
\includegraphics[width=0.5\textwidth]{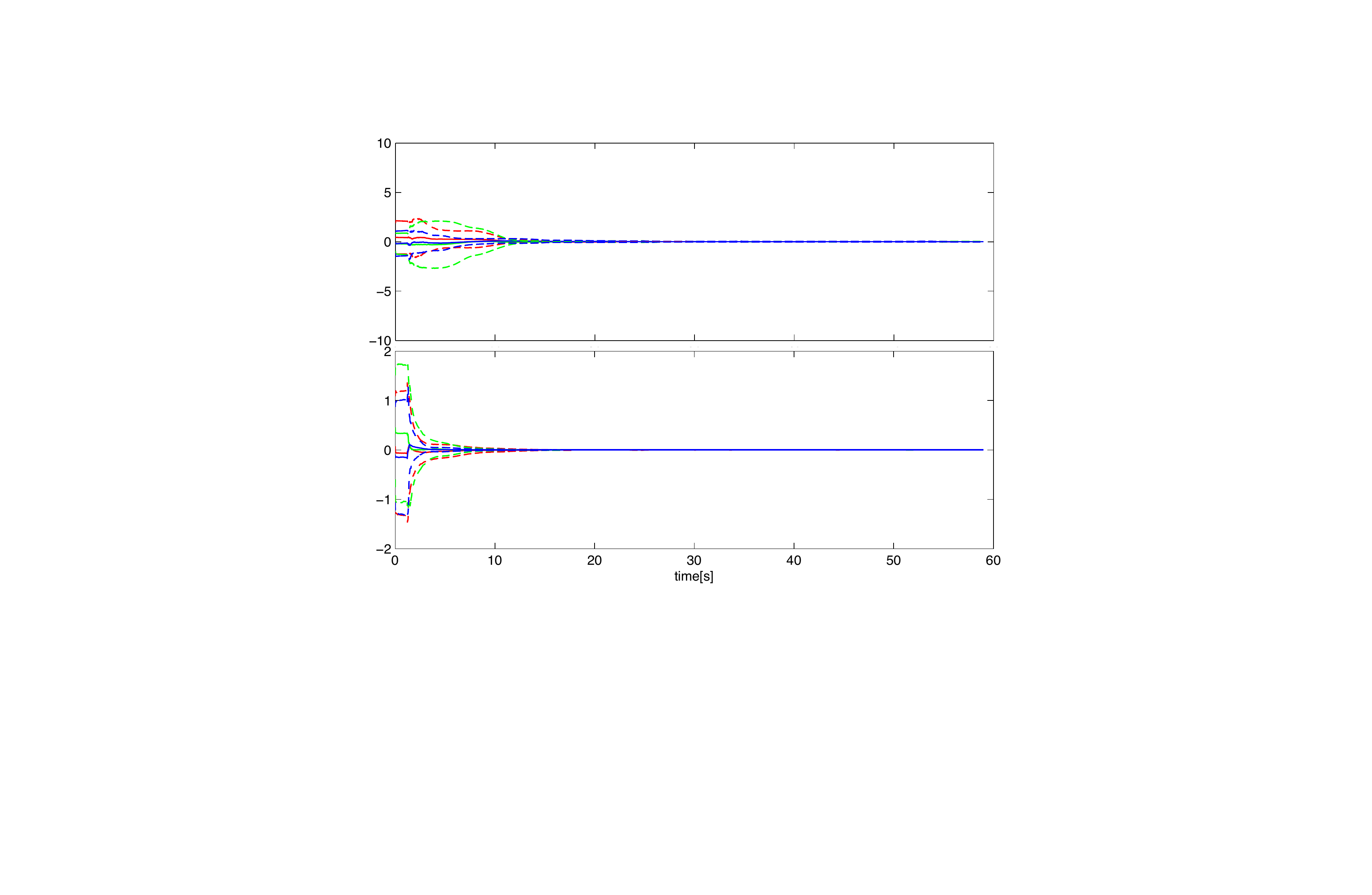}
\caption{\sl Mean (solid line) and twice the standard deviation (dashed lines) of squared estimation errors of alignment parameters (top translational, bottom rotational) aggregated over 50 Monte-Carlo trials with a constant bias.}
\label{fig-mc-statbias}
\end{figure}

\begin{figure}[htb]
\center
\hspace{-0.5cm}
\includegraphics[width=0.5\textwidth]{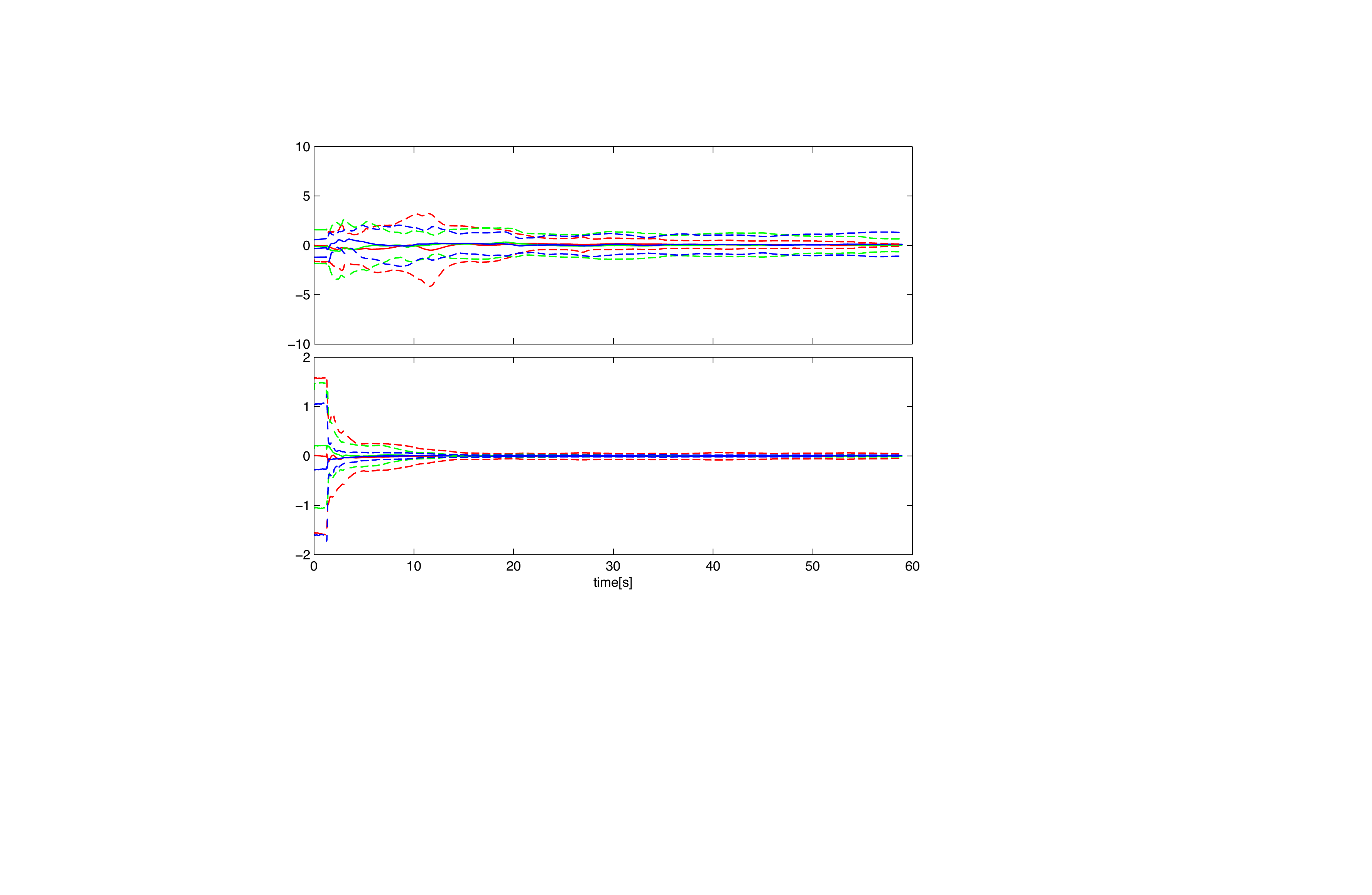}
\caption{\sl Mean (solid line) and twice the standard deviation (dashed lines) of squared estimation errors of alignment parameters (top translational, bottom rotational) aggregated over 50 Monte-Carlo trials with a time-varying bias with similar noise characteristics to the simulated sensors models.}
\label{fig-mc-floatbias}
\end{figure}
The experiments thus confirm the analysis.

\section{Discussion}

We have shown that when inertial sensor biases are included as model parameters in the state of a filter used for navigation estimates, with bias rates treated as unknown inputs, the resulting model is {\em not observable}. 
%That is, the set of indistinguishable states is not a singleton, as one would be led to believe if assuming that bias rates are ``white noise'' that is independent of all other states (including the biases themselves). While the treatment of bias rates as  (possibly small but otherwise) {\em unknown inputs}, as opposed to {\em noise}, is a matter of modeling, empirical evidence dating back to \cite{jonesS09} shows that indeed convergence is not to the ``ground truth'' but to a steady-state error that can be bounded as a function of the input characteristics.

Consequently, we have re-formulated the problem of analyzing the convergence characteristics of (any) filters for vision-aided inertial navigation {\em not} as one of observability or identifiability, but one of {\em sensitivity}, by bounding the set of indistinguishable trajectories to a set whose volume depends on motion characteristics.

The advantage of this approach, compared to the standard observability analysis based on rank conditions, is that we characterize the indistinguishable set explicitly. Furthermore, rank conditions are ``fragile'' in the sense that the model can be nominally observable, and yet the condition number of the observability matrix be so small as to render the model effectively unobservable. We quantify the ``degree of unobservability'' as the sensitivity of the solution set to the input; provided that sufficient-excitation conditions are satisfied, the unobservable set can be bounded and effectively be treated as a singleton. More in general, however, the analysis provides an estimate of the uncertainty surrounding the solution set, as well as a guideline on how to limit it by enforcing certain gauge transformations.

\section*{Acknowledgements}
This work was supported by the Air Force Office of Scientific Research (grant no. AFOSR FA9550-12-1-0364) and the Office of Naval Research (grant no. ONR N00014-13-1-034).

%\onecolumn
%% \appendix
%% \section{Proofs}
%% \label{app-proofs}

%% \centering
%% Proofs available as Supplementary Material.

%\end{document}
\onecolumn
\appendix

\section{Proofs}

\subsection{Definitions}

The mechanization equations can be written in compact notation as 
\be
\begin{cases}
\dot x = f(x) + c(x)u + Dv + n_x \\
y(t) = h(x) + n_y
\end{cases}
\ee
where $x \doteq \{T, R, \V, \omega_b, \alpha_b, T_{cb}, R_{cb}, y^i_{t_i}, Z^i_{t_i}\}$ is the state, $u \doteq \{\omega_{imu}, \alpha_{imu}\} $ the input, and $v \doteq \{ w_b, \xi_b\}$ the unknown input. Note that we are overloading the notation, by using $v$ to denote the unknown input in the compact notation, and $\V$ the translational velocity in the original notation. This should not cause confusion as the two are never used in conjunction. Recall that $y^t \doteq \{y(\tau)\}_{\tau = 0}^t$ is a collection of output measurements, and $x^t = \{x(\tau)\}_{\tau = 0}^t$ a state {\em trajectory}. {In the absence of unknown inputs, $v = 0$, } given output measurements $y^t$ and known inputs $u^t$, we call 
\begin{equation}
{\cal I}(y^t | u^t; \tilde x_0) \doteq \{ {\tilde x}^t \ | \ y^t = h({\tilde x}^t) \ {\rm s. \ t. } \ \dot {\tilde x}(t) = f({\tilde x}) + c({\tilde x}) u(t), \ \tilde x(0) = \tilde x_0  \ \forall \ t\}
\end{equation}
the {\em indistinguishable set}, or set of {\em indistinguishable trajectories}, for a given input $u^t$. If the initial condition $\tilde x_0 = x_0$ equals the ``true'' one, the indistinguishable set contains at least one element, the ``true'' trajectory $x^t$. However, if $\tilde x_0 \neq x_0$, the true trajectory may not even be part of this set.

If the indistinguishable set is a singleton (it contains only one element, $\tilde x^t$, which is a function of the initial condition $\tilde x_0$), we say that the model is {\em observable up to the initial condition}, or simply {\em observable}.\footnote{We will assume that the solution of the differential equation $\dot x = f(x) + c(x) u$ is unique and continuously dependent on the initial condition, so if we impose $\tilde x_0 = x_0$, then $\tilde x^t = x^t$.} {If $\{\tilde x^t\}$ is further independent of the initial condition, we say that the model is {\em strongly observable}:}
$
{\cal I}(y^t | u^t; \tilde x_0) = \{ x^t\} \ \forall \ \tilde x_0, \ u^t. 
$

\iffalse
If the indistinguishable set is not a singleton, but the collection of trajectories form an equivalence class under different initial conditions, we say that the model is {\em observable up to the initial condition}. For instance, we may have that $x^t$ is in the form $x(t) = \phi(t)x_0$ and indistinguishable trajectories are of the form $\tilde x(t) = \phi(t) \tilde x_0$, for the same $\phi(t)$, for all $t$. In this case, if we consider equivalent trajectories that differ solely by their initial condition, we have that the indistinguishable set is made of equivalence classes (under the Gauge transformation defined by the initial conditions), and we say that the model is observable up to the initial condition if the indistinguishable set is a single equivalence class. Note that, in this case, we can fix the initial condition arbitrarily (canonization), say to $\tilde x_0$, and obtain a singleton indistinguishable set that does {\em not} the true trajectory, but one that is related to it by {\em gauge 
transformation}. We will defer the treatment to the initial condition to Sect. \ref{sect-gauge}, and neglect $x_0$ in the meantime. 
\fi

If the state includes unknown parameters with a trivial dynamic, and there is no unknown input, $v = 0$, then observability of the resulting model implies that the parameters are {\em identifiable.} That usually requires the input $u^t$ to be {\em sufficiently exciting} (SE), in order to enable disambiguating the indistinguishable states,\footnote{Sufficient excitation means that the input is {\em generic}, and does not lie on a thin set. That is, even if we could find a particular input $u^t$ that yields indistinguishable states, there will be another input  that is infinitesimally close to it} that will disambiguate them. as the definition does not require that every input disambiguates any state trajectories.

In the presence of {\em unknown inputs} $v \neq 0$, consider the following definition
\begin{equation}
{\cal I}_v(y^t | u^t; \tilde x_0) \doteq \{ {\tilde x}^t \ | \ \exists \ v^t \ {\rm s. \ t. } \ y^t = h(\tilde x^t), \ \dot {\tilde x}(t) = f({\tilde x}) + c({\tilde x}) u(t) + D v(t) \ \forall \ t; \ \tilde x(0) = \tilde x_0\}
\end{equation}
which is the set of {\em unknown-input indistinguishable states}. The model $\{f, c, D\}$ is said to be {\em unknown-input observable} (up to initial conditions) if the unknown-input indistinguishable set is a singleton. If such a singleton is further independent of the initial conditions, the model is strongly observable. The two definitions coincide once the only admissible unknown input is $v^t = 0$ for all $t$.

It is possible for a model to be observable (the indistinguishable set is a singleton), but not unknown-input observable (the unknown-input indistinguishable set is dense). In that case, the notion of {\em sensitivity} arises naturally, as one would want to measure the ``size'' of the unknown-input indistinguishable set as a function of the ``size'' of the unknown input. For instance, it is possible that if the set of unknown inputs is small in some sense, the resulting set of indistinguishable states is also small. If $v \in {\cal V}$ and for any $\epsilon>0$ there exists a $\delta >0$ such that ${\rm vol}({\cal V}) \le \epsilon$ for some measure of volume implies ${\rm vol}({\cal I}_v(y^t | u^t; \tilde x_0)) < \delta$ for any $u^t, \tilde x_0$, then we say that the model is {\em bounded-unknown-input/bounded-output observable} (up to the initial condition). If the latter volume is independent of $\tilde x_0$ we say that model is strongly bounded-unknown-input/bounded-output observable.

\subsection{Preliminary claims}
\label{sect-bearing-analysis}

\begin{lemma}\label{lemma-one}
Given $S\in \SO(3)$ and $\dot S\in T_{\SO(3)}(S)$, 
and $a\in\RR$, the matrix $(aS + \dot S)$ is nonsingular unless $a=0$, in which case it has rank $2$ or $0$.
\end{lemma}
\begin{proof} The tangent $\dot S$ has the form $S M$, where $M$ is some skew-symmetric matrix. \margincut{notation overload $x$}
As such, $Mx\perp x$ for any $x\in\RR^3$,  so
$$\|(aS+\dot S)x\|_2^2 = \|S(aI+M)x\|^2_2 = \|ax\|^2_2 +\|Mx\|^2_2.$$
The above is zero only if $ax=0$, so $(aS+\dot S)$ is nonsingular.
For the remaining cases, observe that a $3\times3$ skew-symmetric matrix has rank 2 or 0.
\end{proof}

\begin{lemma}\label{claim-one}
Let $(R(t),T(t))$ and $(\Rw(t),\Tw(t))$ be differentiable trajectories in $\SE(3)$.
For each time $t\1\in[0,T]$, there exists an open, full-measure subset $\Aa_{t\1}\subs\RR^3$
such that:
\begin{quote}
For any two static point-clouds $\{X^i\}_{i=1}^N\subs\Aa_{t\1}$ and $\{\Xw^i\}_{i=1}^N\subs\RR^3$ that satisfy
\begin{equation}
\pi\bigl(R\inv(t) (X^i-T(t)\bigr) = \pi\bigl(\Rw\inv(t) (\Xw^i-\Tw(t))\bigr) 
\quad\text{for all $i$ and $t$}
\label{prob}
\end{equation}
there exist constant scalings $\sigma_{it\1}>0$ and a constant rotation $S_{t\1}=\Rw({t\1})R\inv({t\1})$ such that
$$ \sigma_{it\1} S_{t\1} (X^i - T(t)) = (\Xw^i-\Tw(t)) + O((t-t\1)^2)
\quad\text{for all $i$ and $t$}.$$
Furthermore, if $T(t\1)\neq0$, then $\sigma_{it\1}=\sigma_{t\1}$ for all $i$.
\end{quote}
\end{lemma}
\begin{proof}
Write $S(t) = \Rw(t)R\inv(t)$.
Equality under the projection $\pi$ implies that there exists a scaling $\sit(t)$
(possibly varying with $X^i$ and $t$) 
such that
\begin{equation}\sit S\bigl(X^i-T)  = \Xw^i-\Tw. \label{finalform}\end{equation}
For a given time $t\1$, 
we wish to find a suitably large set $\Aa_{t\1}$ such that
$\dot\sit(t\1)=\dot S(t\1)=0$ and $\sit(t\1)$ is independent of $X^i$, when $X^i\in\Aa_{t\1}$
Taking time derivatives,
$$
\bigl({\dot{\sit}}S + \sit\dot{S}\bigr)(X^i-T) - \sit S\dot{T}
= - \dot{\Tw}
$$
or, dividing by $\sit$,
\begin{align}
\bigl(\tfrac{\dot{\sit}}{\sit}S + \dot{S}\bigr)(X^i-T) - S\dot{T}
= - \tfrac{1}{\sit}\dot{\Tw}.
\label{compare}
\end{align}
Differentiating both sides with respect to $X^i$,
\begin{equation}
\bigl(\tfrac{\dot{\sit}}{\sit}S + \dot {S}\bigr)\dX^i +
 \bigl(\tfrac{d}{dX^i}\bigl(\tfrac{\dot {\sit}}{\sit}\bigr)\dX^i\bigr) S(X^i-T) =
-\bigl(\tfrac{d}{dX^i}\bigl(\tfrac{1}{\sit}\bigr)\dX^i\bigr)\dot{\Tw}. \label{moneyshot}
\end{equation}
Observe that
$\tfrac{d}{dX^i}\bigl(\tfrac{\dot {\sit}}{\sit}\bigr)\dX^i$ and
$\tfrac{d}{dX^i}\bigl(\tfrac{1}{\sit}\bigr)\dX^i$
are scalars. 
By Lemma \ref{lemma-one}, the LHS 
%of (\ref{moneyshot}) we
has rank 2 or greater (as a linear map on $\delta X^i$), unless $\dot {\sit}(t\1)=0$.  
The right-hand side (RHS), however, has rank at most 1.  
Thus, (\ref{compare})
is invalid for almost all $X^i$, unless $\dot {\sit}(t\1)=0$
(two maps of different ranks can only agree on a submanifold).
Plugging $\dot {\sit}=0$ into (\ref{moneyshot}), we are left with
\begin{equation}
\dot {S}\dX^i =
-\bigl(\tfrac{d}{dX^i}\bigl(\tfrac{1}{\sigma_i}\bigr)\dX^i\bigr)\dot {\Tw}. \label{Double-moneyshot}
\end{equation}
Now, the LHS 
%of (\ref{Double-moneyshot}) 
has rank 2 or 0, while the RHS has rank 1 or 0.
Again, (\ref{compare}) is invalid for almost all $X^i$, unless 
$\dot {S}(t\1)=0$.
Let $\Aa_{t\1}\subs\RR^3$ be the open, full-measure subset \ignore{of full measure} (being the complement of two submanifolds) on which the latter must hold.
If, in addition, $\dot {T(t\1)}\neq 0$, then $\dot{\Tw}(t\1)\neq 0$ and
$\frac{d\sigma_i}{dX^i}(t\1)=0$,
%and for $\{X^i\}\subs\Aa_{t\1}$, 
we can finally write %(\ref{finalform}) as
$$
\sigma_{t\1} S_{t\1} (X^i - T) = \Xw^i-\Tw + O((t-t\1)^2).
$$
\end{proof}

\iffalse
\begin{claim}\label{claim-one}
Let $X^i(t) \in \real^3, i = 1, \dots, N(t); t\in {\mathbb Z}$ and $\Xw^i(t) \neq X^i(t)$. Then $\pi(X^i(t)) = \pi(\Xw^i(t))$ if and only if $\Xw^i(t) = \sigma(t) X^i(t) \ \forall \ i, t$, where $\sigma(t) > 0$ is an arbitrary positive scalar-valued function of $t$ (and $i$).
\end{claim}
This follows directly from the definition of the projection map $\pi$.
\fi

\begin{claim}[Indistinguishable Trajectories from Bearing Data Sequences]\label{claim-two}
Let $g(t)$ and $\gw(t)$ be differentiable trajectories in $\SO(3)$.
There exists an open, full-measure subset $\Aa\subs\RR^3$ such that
\begin{quote}
Given two static, generic (non-coplanar) point clouds $\{X^i\}_{i=1}^N\subs\Aa$ 
and $\{\Xw^i\}_{i=1}^N\subs\RR^3$, satisfying
$$\pi(g^{-1}(t)\, X^i) = \pi(\gw^{-1}(t)\, \Xw^i) \quad \text{for all $i$ and $t$},$$
there exist constant scalings $\sigma_i>0$ and a constant transformation $\bar g\in \SE(3)$ such that
\begin{equation}
\begin{cases}
\begin{tabular}{>{$}r<{$} >{$\!\!\!\!\!}l<{$}}
\Xw^i &= \sit (\bar g X^i) \\
\gw(t) &= \sit (\bar g g(t))
\end{tabular}
\end{cases}
\;\text{for all $i$ and $t$.}
\label{eq-gauge}
\end{equation} 
Furthermore, if $g(t)$ has a non-constant translational component, then $\sigma_i=\sigma$ for all $i$.
\end{quote}
\end{claim}
\begin{proof}
Write $g(t) = (R(t), T(t))$ and $\gw(t) = (\Rw(t), \Tw(t))$.
Let $\Aa = \{X\in\RR^3:\, X\in\Aa_{t\1}\text{ for almost all } t\1\}$, with $\Aa_{t\1}$ defined as in Lemma \ref{claim-one}.  
By Fubini's theorem, this has full measure in $\RR^3$.
If $\{X^i\}\subs\Aa$, then the conditions for
Lemma \ref{claim-one} are satisfied for almost all $t$, and thus
there exist \emph{constant} (being stationary for almost all $t$) scalings $\sit$ and rotation 
$S=\Rw(t)R(t)\inv\in \SO(3)$ such that
$\Xw^i = \sit S(X^i - T_t) + \Tw_t$.

Define $\bar g(t) = (\sit\inv \gw(t))\, g(t)\inv$, and observe that
\begin{align*}
\Xw^i = \sit S(X^i - T_t) + \Tw_t
= \sit(\Rw_t(g\inv X^i) + \sit \inv \Tw_t)
=\sit \bigl((\sit \inv \gw(t))\, g(t)\inv X^i\bigr)
=\sit (\bar g(t) X^i).
\end{align*}
If this affine relation holds for the generic set $\{X^i\}$, then $\bar g(t)$ must be constant.  Next,
\begin{align*}
\sit (\bar g g(t)) = \sit((\sit \inv\gw(t))\, g(t)\inv g(t)) = \sit(\sit \inv\gw(t)) = \gw(t).
\end{align*}
Finally, if $T(t\1)=0$ for some $t\1$, then $\sit = \sit(t\1)=\sigma(t\1)=\sigma$
for all $i$.
\end{proof}

In what follows, we will avoid the cumbersome discussion of sets such as $\Aa\subs\RR^3$,
defined by a given trajectory,
and will instead speak of \emph{sufficiently exciting} trajectories, 
for which a given point cloud is suitable for tracking.
\begin{defn}[Sufficiently Exciting Motion]
A trajectory $g(t)$ is {\bf{sufficiently exciting}} relative to a point-\break cloud $\{X^i\}_{i=1}^N\subs\RR^3$ if,
for all $\{\Xw^i\}_{i=1}^N\subs\RR^3$ and $\gw(t)$ in $\SE(3)$,
\begin{align}
\pi(g(t)\inv(t)X^i) &= \pi(\gw(t)\inv \Xw^i) 
\quad\text{for all $i$ and $t$}\label{def_suffex}
\iff
\\
&
\left(
\begin{tabular}{>{$}r<{$}>{$\!\!\!\!\!}l<{$}}
\Xw^i &= \sigma(\bar g X^i)\\
\gw(t) &= \sigma(\bar g g(t))
\end{tabular}
\;\text{for all $i$ and $t$}
\right)
\;\text{for some constant $\sigma>0$ and $\bar g\in\SE(3)$.}\notag
\end{align}

That is, if the projection map $\pi(g(t)X^i)$ defines $g(t)$ and $\{X^i\}$ up to a constant rotation and mapping.
\end{defn}
\noindent
Observe that the right-to-left implication is always true:  if the RHS holds, then
$$\pi(\gw(t)\inv\Xw^i) = \pi((\sigma \bar g g(t))\inv \sigma(\bar g X^i)) 
\pi(g(t)\inv \bar g\inv \sigma\inv \sigma \bar g X^i )
= \pi(g(t)\inv X^i).$$
We will see that the sufficient excitation condition is very easily satisfied.
\begin{claim}
Given trajectories $g(t)$ and $\gw(t)$ in $\SE(3)$ with non-constant translation, 
and a set $\{X^i\}_{i=1}^N$ of $N\geq 4$ points sampled independently from identical distributions (i.i.d.) over $\RR^3$,
the trajectory $g(t)$ is a.s. sufficiently exciting relative to $\{X^i\}$.
\end{claim}
\begin{proof}
Fix $g(t)$. By Claim \ref{claim-two},
there exists a full-measure $\Aa\subs\RR^3$ such that (\ref{def_suffex}) holds for any 
static, generic point clouds
$\{X^i\}_{i=1}^N\subs\Aa$ and 
$\{\Xw^i\}_{i=1}^N\subs\RR^3$.
If $\{X^i\}$ is sampled i.i.d. from a non-singular distribution over $\RR^3$, 
then $\{X^i\}\subs\Aa$ almost surely.
\end{proof}

Equation (\ref{eq-gauge}) establishes the fact that the indistinguishable trajectories are an equivalence class parameterized by a group $\sigma(\bar g)$, called a {\em gauge transformation}. We now include a constant reference frame $g_{a}$. We then have the following claim.
\begin{claim}[Indistinguishable Alignments]\label{claim-thre}
For a point cloud $\{X^i\}_{i=1}^{N(t)}$, $N(t) > 3$, in general position (non-coplanar), and sufficiently exciting motion, 
\begin{equation}
\pi(g_{a} g^{-1} (t) X^i) = \pi(\gw_{a} \gw^{-1} (t) \Xw^i)
\end{equation}
if and only if there exist constants $\sigma >0$, $g_A$ and $g_B\in \SE(3)$ such that
\begin{equation}
{\begin{cases}
\begin{tabular}{>{$}r<{$} >{$\!\!\!\!\!}l<{$}} 
\Xw^i &= \sigma(g_B X^i)\\
\gw(t) &= \sigma(g_B g(t) g_A)\\
\gw_{a} &= \sigma(g_{a} g_A). 
\end{tabular}
\end{cases}}
\end{equation}
\end{claim}
\begin{proof}
From Claim \ref{claim-two} we get constant $g_B\in\SE(3)$ and $\sigma>0$ such that $\Xw^i = \sigma(g_B X^i)$
and 
\begin{eqnarray}
\gw(t)\gw\inv_a = \sigma(g_Bg(t)g_a\inv)
\end{eqnarray}
Let $g_A = g_a\inv\sigma\inv(\gw_a)$.  Then $\gw_a = \sigma(g_ag_A)$ and
$$\gw(t) = \sigma(g_Bg(t)g_A).$$
\end{proof}
We now include groups of points, each with its own reference frame.
\begin{claim}[Indistinguishable Groups] \label{claim-four}
For a number $i = 1, \dots, K$ of groups each with a number $j = 1, \dots, N_i \ge 3$ of points in general position (non-coplanar), and sufficiently exciting motion, 
\begin{equation}
\pi(g_{a} g^{-1} (t) g_i g_a^{-1} X^j) = \pi(\gw_{a} \gw^{-1} (t) \gw_i \gw_a^{-1} \Xw^j)
\end{equation}
if and only if there exist constants $\sigma >0, g_A, g_B,\bar  g_i \in \SE(3)$  such that
\begin{equation}
{\begin{cases}
\begin{tabular}{>{$}r<{$} >{$\!\!\!\!\!}l<{$}}
\Xw^j &= \sigma(g_a \bar g_i^{-1} g_i g_a^{-1} X^j)\\
\gw(t) &= \sigma(g_B g(t) g_A) \\
\gw_i &=  \sigma(g_B \bar g_i g_A) \\
\gw_{a} &= \sigma(g_a g_A)
\end{tabular}
\end{cases}}
\label{eq-vis-only}
\end{equation}
\end{claim}
\begin{proof}
From Claim \ref{claim-two}, we get constant $g_C\in SE(3)$ and $\sigma>0$ such that 
\begin{gather}
\Xw^i=\sigma(g_CX^i), \label{rel x}\\
\gw_a \gw_i\inv \gw(t)\gw_a\inv = \sigma(g_Cg_ag\inv_ig(t)g_a\inv). \label{rel g}
\end{gather}
Define
$$
g_A \doteq g_a\inv\sigma\inv(\gw_a),
\qquad\qquad
g_B \doteq\sigma\inv(\gw_i\g_a\inv)g_Cg_ag_i\inv,
\qquad\qquad
\bar g_i \doteq g_ig_a\inv g_C\inv g_a. 
$$
Then, applying the definition of $\bar g_i$ to (\ref{rel x}),
\begin{align*}
\Xw^j &= \sigma(g_C X^j)
= \sigma((g_a\bar g_i\inv g_i g_a\inv) X^j).
\end{align*}
Applying the definitions of $g_A$ and $g_B$ to (\ref{rel g}),
\begin{align*}
\gw(t) &= \gw_i\gw_a\inv\sigma(g_Cg_ag_i\inv g(t)g_a\inv)\gw_a
= \sigma\bigl(
\underset{g_B}{\underbrace{\sigma\inv(g_i \gw_a\inv) g_C g_a g_i\inv}}\,
g(t)\,
\underset{g_A}{\underbrace{g_a\inv \sigma\inv(\gw_a)}}
\bigr)
= \sigma(g_B g(t) g_A).
\end{align*}
Rearranging the definitions of $g_A$, $g_B$ and $\bar g_i$,
\begin{align*}
\gw_i &= \sigma(g_B g_i g_a\inv g_C\inv)\gw_a
=\sigma\left(g_B g_i g_a\inv g_C\inv \sigma(\gw_a)\right)
=\sigma\bigl(g_B 
\underset{\bar g_i}{\underbrace{g_i g_a\inv g_C\inv g_a}}\,
\underset{g_A}{\underbrace{g_a\inv \sigma(\gw_a)}}
\bigr)
=\sigma(g_B\bar g_i g_A).
\end{align*}
Finally, rearrange the definition of $g_A$ to get
$$\gw_a = \sigma(g_ag_A).$$
\end{proof}
Eq. (\ref{eq-vis-only}) describes the ambiguous state trajectories if only bearing measurement time series are given. In that case, there is no alignment to other sensor, so we can assume without loss of generality that $g_a = Id$ and so for $\gw_a$, which in turn implies $g_A = Id$. The resulting ambiguity is well-known \cite{soatto97} and shows that scale $\sigma$ is constant but arbitrary, that the global reference frame is arbitrary (since $g_B$ is), and that the reference frame of each group is also arbitrary (since $\bar g_i$ is). To lock these ambiguities, we can fix three directions for each group (thus fixing $\bar g_i$) and, in addition, for one of the groups fix the pose (thus fixing $g_B$); finally, we can impose that the centroid of the points in that one group (the ``reference group'') be one, which fixes $\sigma$. Thus, an observer designed based on the standard model, where $3$ directions within each group are saturated, and where the pose of one group is fixed, and the centroid of 
the 
group is fixed, is observable, and under the usual assumptions it should converge to a state trajectory that is related to the true one by an arbitrary unknown scaling, and global reference frame.

Now, when inertial measurements are present, of all the possible trajectories that are indistinguishable from the measurements, we are interested {\em only} in those that are compatible with the dynamical model driven by IMU measurements. Since the fact that $X^j$ and $g_a$ are constant has already been enforced, the model will impose no constraints on $\Xw^j, \gw_i$ and $\gw_a$. However, it will offer constraints on $\gw(t)$, that depends on the arbitrary constants $\sigma, g_A, g_B$. 

\subsection{Indistinguishable trajectories in bearing augmentation}

%%%%%%
\begin{defn}
For an $\RR^3$-valued trajectory $f:\RR\to\RR^3$ and interval $\I\subs\RR^+$, define \margincut{notation overload $f$, $\cal I$, $\times$, $x$}
\begin{align*}
\m(f\!:\!\I)
&\doteq \!\!\inf_{\|x\|=1}\Bigl(\sup_{t\in \I}\,|{f}(t)\cdot x|\Bigr)
= \inf_{\|x\|=1}\Bigl(\sup_{t\in \I}\,\|{f}(t)\times x\|\Bigr),\\
\M(f\!:\!\I)
&\doteq \!\!\sup_{\|x\|=1}\Bigl(\sup_{t\in\I}\,|{f}(t)\cdot x|\Bigr) = \sup_{t\in\I}\|{f}(t)\|
,
\quad\text{and}\\
%\quad
\bar\m(f\!:\!\I) &\doteq \sqrt{\max\{0,\, 2\m(f\!:\!\I)^2 - \M(f\!:\!\I)^2\}}.
\end{align*}
\end{defn}
Observe that $\M(f\!:\!\I) \geq \m(f\!:\!\I) \geq \bar\m(f\!:\!\I)$, and that the inequalities are strict unless
$\{\pm f(t)|\,t\in\I\}$ is dense on the sphere of radius $\M(f\!:\!\I)$.  We use these ``minimum-excitation''
bounds in order to prove a partial converse of the Cauchy-Schwarz inequality:
\begin{lemma}
Let $A = c_1I + c_2R$, for some rotation $R\in\SO(3)$ and scalars $c_1$ and $c_2$. \margincut{notation overload $c$}
Then, for any trajectory $f:\RR^+\to\RR^3$ and set of times $\I\subs\RR^+$,
$$\sup_{t\in \I} \left\|A f(t)\right\| \;\geq\;
\left\|A\right\|\bar\m(f\!:\!\I).
$$
\end{lemma}
\begin{proof}
First, observe that $A$ is orthogonal:
$$AA^T = (c_1I + c_2R)(c_1I + c_2R^T) = 2c_1c_2I + c_1c_2(R + R^T)
= A^TA.$$
Let $\{(\lambda_i, v_i)\}_{i=1}^3$ be orthonormal eigenvalue/eigenvector pairs of $A$, with $\lambda_1\geq\lambda_2\geq\lambda_3$.
\margincut{notation overload $v$}
\begin{align*}
\|A f(t)\|^2 &= \lambda_1^2(v_1\cdot f(t))^2 + \lambda_2^2(v_2\cdot f(t)^2 + \lambda_3^2(v_3\cdot f(t))^2\\
&\geq \lambda_1^2\bigl((v_1\cdot f(t))^2 - (v_2\cdot f(t)^2 - (v_3\cdot f(t))^2\bigr)\\
&=\|A\|^2\bigl(2(v_1\cdot f(t))^2 - \|f(t)\|^2\bigr).
\end{align*}
Taking the supremum over $\I$,
\begin{align*}
\sup_{t\in\I}\|A f(t)\|^2 &\geq \|A\|^2\sup_{t\in\I}\bigl(2(v_1\cdot f(t))^2 - \|f(t)\|^2\bigr)\\
&\geq \|A\|^2\bigl(2\sup_{t\in\I}(v_1\cdot f(t))^2 - \sup_{t\in\I}\|f(t)\|^2\bigr)\\
&\geq \|A\|^2\bigl(2\m(f\!:\!\I)^2 - \M(f\!:\!\I)^2\bigr)
\end{align*}
\end{proof}

\begin{lemma}
Let $A = I-R$, for some rotation $R\in\SO(3)$.
Then, for trajectory $f:\RR^+\to\RR^3$ and $\I\subs\RR^+$,
$$\sup_{t\in \I} \left\|A f(t)\right\| \;\geq\;
\left\|A\right\|\m(f\!:\!\I).
$$
\end{lemma}
\begin{proof}
Let $\{(\lambda, v_1), (\bar\lambda, v_2), (1,0)\}$ be the orthonormal eigenvalue/eigenvector pairs of $R$.
Since $R$ and $I$ commute, $\{(\lambda-1, v_1),$ $(\bar\lambda-1, v_2), (0, u)\}$ are the eigenpairs of $A$, 
and $\|A\|=|\lambda-1| = |\bar\lambda-1|$. Then,
$$\|Af(t)\|^2 = |\lambda-1|^2(v_1\cdot f(t))^2 + |\bar\lambda-1|^2(v_2\cdot f(t))^2 + 0
= \|A\|^2(w\cdot f(t))^2,
$$
where 
$$w 
\doteq \frac{(v_1\cdot f(t)) v_1 + (v_2\cdot f(t)) v_2}{\|(v_1\cdot f(t)) v_1 + (v_2\cdot f(t)) v_2\|}
= \frac{(v_1\cdot f(t)) v_1 + (v_2\cdot f(t)) v_2}{\sqrt{(v_1\cdot f(t))^2 + (v_2\cdot f(t))^2}}.
$$
Taking the supremum over $\I$,
$$\sup_{t\in\I}\|Af(t)\|^2  = \|A\|^2\,\sup_{t\in\I}\|w\cdot f(t)\|^2 \geq \|A\|^2\m(f\!:\I)^2.$$
\end{proof}

The following re-states Claim \ref{claim-five} for completeness: 
\setcounter{claim}{0}
\begin{claim}[Indistinguishable Trajectories from IMU Data]
Let $g(t)= (R(t), T(t)) \in \SE(3)$ be such that
\begin{equation}
\begin{cases}
\begin{tabular}{>{$}r<{$} >{$\!\!\!\!\!}l<{$}}
\dot R &= R(\hw\imu  - \hw_b) \\
\dot T &= \V \\
\dot \V &= R(\alpha\imu  - \alpha_b) + \gamma
\end{tabular}
\end{cases}
\end{equation}
for some known constant $\gamma$ and functions $\alpha\imu (t)$, $\w\imu (t)$ and for some unknown functions $\alpha_b(t), \w_b(t)$ that are constrained to have $\| \dot \alpha_b(t) \| \le \epsilon$, $\| \dot \w_b(t) \| \le \epsilon$, and $\|\ddot\w_b(t)\|\le\epsilon$ at all $t$,
for some $\epsilon<1$.

Suppose $\gw(t) \doteq \sigma(g_B g(t) g_A)$ for some  constant $g_A = (R_A, T_A)$, $g_B = (R_B, T_B)$, $\sigma > 0$,
with bounds on the configuration space such that
$\|T_A\|\leq M_A$ and $0<m_\sigma\leq|\sigma|\leq M_\sigma$.
Then, under sufficient excitation conditions, 
$\gw(t)$ satisfies  (\ref{eq-model-dyn}) if and only if 
\begin{gather}
\| I - R_A \|  \leq  \frac{2{\epsilon}}{\m(\dot{\w}\imu\!:\!{\RR^+})}  \\
|\sigma - 1|  \le \frac{k_{1}\epsilon + M_\sigma\|I-R_A\|}{\m(\dot\alpha\imu\!:\!{\I_{1}})} \\
\|T_A\|\leq \frac{\epsilon(k_{2}+(2M_\sigma+1)M_A)}{m_{\sigma}\,\m(\ddot\w\imu\!:\!{\I_{2}})} \\
\|(1-R_B^T)\gamma\|\leq\frac{\epsilon(k_{3} + M_\sigma M_A) + (|\sigma-1|+\epsilon)\M(\w\imu-\w_b\!:\!\I_{3})\|\gamma\|}
{m_{\sigma}\,\m(\w\imu-\w_b\!:\!\I_{3})}
\end{gather}
for $\I_i$ and $k_i$ determined by the sufficient excitation conditions.
\end{claim}
\setcounter{claim}{6}
\begin{proof}
\begin{description}
 \item
 \item[(\ref{constraint1})] The ambiguous rotation  $\tilde R$ must satisfy $\dot{\tilde R} = \tilde R(\hw\imu - \widehat{\ww}_b)$ for some $\ww_b$:
\begin{align*}
\dot{\tilde R} &= R_B R (\hw\imu  - \hw_b) R_A 
= \tilde R R_A^T(\hw\imu  - \hw_b) R_A = \tilde R (\widehat{R_A^T \w}\imu - \widehat{R_A^T \w}_{b})\\
&= \tilde R (\hw\imu  - \,[\hw\imu  + \widehat{R_A^T \w}\imu - \widehat{R_A^T \w}_{b}]) \nonumber
\end{align*}
where the quantity in brackets is $-\widehat{\ww}_b$, which defines
\begin{equation}
 \ww_{b} \doteq R_A^T\w_b + (I - R_A^T)\,\w\imu. \label{rotdef}
\end{equation}
Taking derivatives and rearranging, 
$$
2\epsilon\geq \|\dot\ww_b-R_A^T\dot\w_b\| = \|(I-R_A^T)\dot\w\imu\|
$$
Since this is true for all $t\in\RR$, we can write
\begin{align*}
2\epsilon&\geq \sup_{t\in\RR}\|(I-R_A^T)\,\dot{\w}\imu(t)\|
\geq \|I-R_A^T\|\,\m(\dot{\w}\imu\!:\!\RR^+).
\end{align*}
%Since $R_A$ is close to the identity, we can write $I-R_A = \w_A + O(\|I-R_A\|^2)$, so the above can be written as
This rearranges to give (\ref{constraint1}).

\item[(\ref{constraint2})]
The ambiguous translation $\tilde T$ must satisfy the dynamics in (\ref{eq-model-dyn}):
\begin{align*}
\ddot{\tilde T} &= \dot{\tilde \V}  = {\tilde R}(\alpha\imu - \tilde\alpha_b) + \gamma
= R_B R R_A(\alpha\imu - \tilde\alpha_b) + \gamma.
\end{align*}
Alternatively, working with $\tilde T = \sigma R_B(R T_A + T)$ and applying the dynamics to $T$,
\begin{align*}
\ddot{\tilde T} & = \sigma R_B(\ddot RT_A + \ddot T)
= \sigma R_B(\ddot RT_A + R (\alpha\imu - \alpha_b) + \gamma).
\end{align*}
Taking the difference between these two expressions,
\begin{align*}
0 &= \sigma R_B\ddot R T_A
  + R_BR(R_A\tilde\alpha_b - \sigma\alpha_b)
  + R_BR(\sigma\alpha\imu - R_A\alpha\imu)
  + (\sigma R_B - I)\gamma,
\end{align*}
and multiplying by $R^TR_B^T$,
\begin{align*}
0&= \sigma(R^T\ddot R)T_A
  + (R_A\tilde\alpha_b - \sigma\alpha_b)
  + (\sigma\alpha\imu - R_A\alpha\imu)
  + R^T(\sigma - R_B^T)\gamma\\
&= \sigma((\hw\imu-\hw_b)^2 + (\dot\hw\imu-\dot\hw_b))T_A
  + (R_A\tilde\alpha_b - \sigma\alpha_b)
  + (\sigma\alpha\imu - R_A\alpha\imu)
  + R^T(\sigma - R_B^T)\gamma.
\end{align*}
Differentiating again,
\begin{align}
0 &= \sigma(\dot R^T\ddot R + R^T\dddot R)T_A\label{const1}\\
  &+ ((I-R_A)\sigma + (\sigma-1)R_A)\dot\alpha\imu\label{const2}\\
  &+ \dot R^T((I-R_B^T)\sigma + (\sigma-1)R_B^T)\gamma.\label{const3}\\
  &+ (R_A\dot{\tilde\alpha}_b - \sigma\dot{\alpha}_b)\label{const4}
\end{align}
As a sufficient excitation condition, assume that
$\|\dot R(t)\|\leq\epsilon$, $\|\ddot R(t)\|\leq\epsilon$, and
$\|\dddot R(t)\|\leq\epsilon$, for 
$t\in\I_{1}$.
Under these constraints, (\ref{const2}) is
bounded by $k_{1}\epsilon$, where, e.g. $k_{1} \doteq 2M_\sigma M_A + (2M_\sigma+1)(\|\gamma\|+1)$.
In that case,
\begin{align*}
k_{1}\epsilon &\geq \max_{t\in\I_{1}}\|((I-R_A)\sigma + (\sigma-1)R_A)\dot\alpha\imu(t)\|\\                                                                                        
&\geq |\sigma-1|\m(\dot\alpha\imu\!:\!{\I_{1}}) - M_\sigma\|I-R_A\|.
\end{align*}
This rearranges to give (\ref{constraint2}). 
\item[(\ref{constraint3})]
Now, assume that
$\|\dot R(t)\|\leq\epsilon$, $\|\ddot R(t)\|\leq\epsilon$, and
$\|\ddot T(t) - \gamma\|\leq\epsilon$, for $t\in\I_{2}$.  Under these constraints, $\|\dot\alpha_{\imu}\|\leq2\epsilon$, and
(\ref{const1}) is bounded by $k_{2}\epsilon$, where, e.g. $k_{2}\doteq (2M_\sigma+1)(\|\gamma\| + 3)$.
In that case,
\begin{align*}
 k_{2}\epsilon&\geq\max_{t\in\I_{2}}\|\sigma ((\hw\imu-\hw_b)(\dot\hw\imu-\dot\hw_b) + (\ddot\hw\imu-\ddot\hw_b)) T_A\|\\
 &=\max_{t\in\I_{2}}\|\sigma ((R^T\dot R)(R^T\ddot R - (R^T\dot R)^2) + (\ddot\hw\imu-\ddot\hw_b)) T_A\|\\
 &\geq m_{\sigma}\,\!\max_{t\in\I_{2}}\|\ddot\w\imu(t)\times T_A\| - (2M_\sigma+1) M_A\epsilon\\
 &\geq m_\sigma\,\|T_A\|\,\m(\ddot\w\imu\!:\!{\I_{2}}) - (2M_\sigma+1) M_A\epsilon.
\end{align*}
This rearranges to give (\ref{constraint3}).
\item[(\ref{constraint4})]
Finally, assume that $\|\ddot R(t)\|\leq\epsilon$, $\|\dddot R(t)\|\leq\epsilon$, and $\|\ddot T(t) - \gamma\|\leq\epsilon$
for $t\in\I_{3}$.
As before, $\|\dot\alpha\imu\|\leq2\epsilon$.  Then, (\ref{const1}) + (\ref{const2}) is bounded by $k_{3}\epsilon$, where, e.g. $k_{3}=2M_\sigma+3$.
In that case,
\begin{align*}
k_{3}\epsilon &\geq
   \|\sigma(\dot R^T\ddot R + R^T\dddot R)T_A
  + \dot R^T((I-R_B^T)\sigma + (\sigma-1)R_B^T)\gamma\|\\
&\geq \|\sigma \dot R^T (\ddot R + (I-R^T_B))\gamma\| - M_\sigma M_A\epsilon - |\sigma-1|\,\|\dot R^T\|\,\|\gamma\|\\
&\geq m_{\sigma}\,\|\dot R^T (I-R^T_B)\gamma\| - M_\sigma M_A\epsilon - (|\sigma-1|+\epsilon)\,\|\dot R^T\|\,\|\gamma\|\\
&\geq m_{\sigma}\,\m(\dot R^T\!:\!\I_{3})\|(1-R_B^T)\gamma\|
- \epsilon(k_{3} + M_\sigma M_A) - (|\sigma-1|+\epsilon)\M(\dot R^T\!:\!\I_{3})\|\gamma\|
\end{align*}
This rearranges to give (\ref{constraint4}).
\end{description}
\end{proof}

\section{Initialization and observability of gravity}
\label{app-gravity}

The initialization described above assumes that the initial orientation $R(0)$ is fixed in such a way that, in the body frame at time zero, gravity has the form $[0, \ 0, \| \gamma \|]^T$. In other words, it is assumed that the body frame at time zero coincides with the spatial frame, {\em i.e.}, it is aligned with gravity. To accomplish this, accel measurements are averaged, and the initial condition is chosen to align the average with one ordinate axis. Unfortunately, however, the accel measurements include (yet uncompensated) biases, that are therefore averaged along with gravity, resulting in a misalignment of gravity.

Therefore, the question remains as to whether such a misalignment and the resulting acceleration can be ``absorbed'' by the motion states $(R(t), T(t))$, or whether the error ends up polluting other states, in particular the bias estimates. To answer the question, we write the imaging model relative to the world frame $w$, the body frame at time $t=0$, $b_0$, the body frame at time $t$, $b_t$, and the camera frame $c$ as 
\be
y = \pi\left( g_{c b_t} g_{b_t b_0} g_{b_0 w} p_w \right) \doteq \pi\left( g_a  g^{-1} (t) g_0^{-1} p \right) 
\ee
where we have used the short-hand notation adopted in this appendix on the right-hand side. If the spatial frame is aligned with gravity, we would have $g_0 = e$, but otherwise $g_0$ can be arbitrary (although close to the identity). If we neglect $g_0$, rather than inferring $g(t)$, we would be inferring $\tilde g(t) = g_0 g(t)$, whose translational component $\tilde T$ would produce accelerations different from $T$, via
\be
\ddot{\tilde T} = \tilde R (\alpha - \alpha_b) + R_0 \gamma \doteq  \tilde R (\alpha - \alpha_b) + \gamma + \underbrace{(R_0 - I) \gamma} 
\ee
resulting in an uncompensated bias in brackets, which is then integrated into the motion estimate $\tilde T(t)$, inducing drift. This would be particularly visible when the platform is kept still, so the left-hand side is zero, and the bias has to absorb the residual acceleration:
\be
\tilde \alpha_b = \alpha_b + \tilde R^T(R_0 - I)\gamma.
\ee
Thus, {\em one cannot simply initialize the body pose as if it coincided with the spatial frame} and expect that pose would align, because that causes a component of gravity, dependent on the gyro bias, to be absorbed either in the motion estimates (during motion) or in the biases (during stationary or constant velocity segments). 

Two possible ways to address this issue are to (a) {\em not} assume that $\gamma$ is known, and instead insert it in the state with a pseudo-measurement equation enforcing its norm (thus effectively estimating the direction of gravity in the body frame at time zero as part of the inference process), or (b) to reduce the enforcement of the gauge ambiguity so two degrees of freedom in the rotational component of the pose are allowed to float.

The latter can be addressed by fixing {\em two directions}, rather than three, in the gauge transformation. To see that, recall from previous claims that $R(t)$ is observable up to a rotation by $\theta$ about gravity, and $T(t)$ is observable up to a constant frame. Imposing two directions $y^1, y^2$ to be constant, we obtain
\bea
e^{-\gamma \theta} (\bar y^1 Z^1 - T_0) &=& e^{-\gamma \tilde \theta} (\bar y^1 \tilde Z^1 - \tilde T_0) \\
e^{-\gamma \theta} (\bar y^2 Z^2 - T_0) &=& e^{-\gamma \tilde \theta} (\bar y^2 \tilde Z^2 - \tilde T_0) 
\eea
from which 
\be
e^{-\gamma (\theta-\tilde \theta)} [\bar y^1, \ - \bar y^2] \ba{c} Z^1 \\ Z^2\ea =  [\bar y^1, \ - \bar y^2] \ba{c} \tilde Z^1 \\ \tilde Z^2\ea
\ee
where the right-hand side is any vector on the plane spanned by $\langle  \bar y^1, \bar y^2 \rangle$ and, for the left-hand side to equal the right-hand side with $\tilde \theta \neq \theta$ and $\tilde Z^i \neq Z^i$ we must have that the matrix $e^{-\gamma(\theta  - \tilde \theta)}$ leaves said plane invariant. This can only happen if $\gamma \perp \langle  \bar y^1, \bar y^2 \rangle$.  This can happen, for instance if the two reference directions are chosen along the horizon. Therefore, care must be exercised to avoid this degenerate case. 

If gravity is inserted in the state, the observability analysis conducted in the body of the paper must be amended. Using the results of Claim \ref{claim-thre}, we have that $\tilde g = \sigma(g_B g g_A)$ and therefore the ambiguous acceleration is given by
\be
\ddot{\tilde T} = R_B \ddot R \sigma T_A + R_B \sigma(R(\alpha - \alpha_b) + \gamma) = R_B R R_A (\alpha - \tilde \alpha_b) + \tilde \gamma
\ee
where now $\tilde \gamma$ is allowed to differ from $\gamma$, provided that $\| \tilde \gamma \| = \| \gamma \|$. From the above equation we have
\be
\tilde \gamma =  R_B \sigma \gamma - R_B R \underbrace{(R_A - \sigma I)}_{=0}\underbrace{\alpha}_{\uparrow} + R_B \underbrace{\ddot R}_{\uparrow} \sigma \underbrace{T_A}_{=0} + \sigma R_B R \alpha_b - R_B R R_A \tilde \alpha_b
\ee
where the arrow indicates quantities that can vary arbitrarily under sufficient excitation conditions, and therefore enforce the bracketed quantities to be zero, which yields
\be
\tilde \gamma = R_B \gamma + R_B \underbrace{R}_{\uparrow}\underbrace{(\tilde \alpha_b - \alpha_b)}_{=0} = R_B \gamma
\ee
which enforces $R_B$ to be of the form $R_B = e^{\gamma \theta}$, but otherwise showing that the direction of gravity in the body frame at time zero is observable. Since $\theta$ is unconstrained, as well as $T_B$, we still have to fix four degrees of freedom, for instance two directions not spanning a plane normal to gravity, as done in the previous case. Note that this argument is valid only if the bias $\alpha_b$ is constant; otherwise, a significantly more involved analysis is necessary.

\end{document}